\pdfoutput=1
\documentclass{article}
\usepackage{fullpage}
\usepackage{layout}
\usepackage{multirow}

\usepackage{amsfonts}
\usepackage{amsmath}
\usepackage{amsthm}
\usepackage{amssymb} 
\usepackage{dsfont}

\usepackage{nicefrac}
\usepackage{mathtools}

\usepackage{mdframed} 
\usepackage{thmtools,thm-restate}

\definecolor{shadecolor}{gray}{0.90}
\declaretheoremstyle[
headfont=\normalfont\bfseries,
notefont=\mdseries, notebraces={(}{)},
bodyfont=\normalfont,
postheadspace=0.5em,
spaceabove=1pt,
mdframed={
  skipabove=8pt,
  skipbelow=8pt,
  hidealllines=true,
  backgroundcolor={shadecolor},
  innerleftmargin=4pt,
  innerrightmargin=4pt}
]{shaded}

\declaretheorem[style=shaded,within=section]{definition}
\declaretheorem[style=shaded,sibling=definition]{theorem}
\declaretheorem[style=shaded,sibling=definition]{proposition}
\declaretheorem[style=shaded,sibling=definition]{assumption}
\declaretheorem[style=shaded,sibling=definition]{corollary}

\declaretheorem[style=shaded,sibling=definition]{lemma}
\declaretheorem[style=shaded,sibling=definition]{remark}

\usepackage{algorithm}
\usepackage[noend]{algpseudocode}

 \usepackage{xcolor}
\usepackage{color}
\usepackage{graphicx}
\graphicspath{{./figures/}}  

\usepackage{nameref,hyperref,cleveref}

\newcommand{\R}{\mathbb{R}} 
\newcommand{\N}{\mathbb{N}} 

\newcommand{\cD}{{\cal D}}

\newcommand{\cO}{{\cal O}}

\newcommand{\cW}{{\cal W}}

\newcommand{\mD}{{\bf D}}

\newcommand{\mI}{{\bf I}}

\newcommand{\zeros}{{\bf 0}}

\usepackage[colorinlistoftodos,bordercolor=orange,backgroundcolor=orange!20,linecolor=orange,textsize=scriptsize]{todonotes}

\newcommand{\eqdef}{\coloneqq} 

\newcommand{\dotprod}[1]{\left< #1\right>} 
\newcommand{\norm}[1]{ \left\| #1 \right\|}      

\DeclareMathOperator{\argmin}{argmin}        

\newcommand{\E}[1]{\mathbb{E}\left[#1\right] } 
\newcommand{\EE}[2]{\mathbb{E}_{#1}\left[#2\right] }

\usepackage[firstinits=true,backend=bibtex,style=alphabetic,citestyle=alphabetic, url =false, arxiv =false, isbn = false, doi = false]{biblatex}

\bibliography{references}

\title{Function Value Learning: Adaptive Learning Rates Based on the Polyak Stepsize and Function Splitting in ERM}
\author{Guillaume Garrigos \\ 
guillaume.garrigos@lpsm.paris\\
Université Paris Cité and Sorbonne Université, CNRS\\
       Laboratoire de Probabilités, Statistique et Modélisation\\
       F-75013 Paris, France\\
\and Robert M. Gower \\
rgower@flatironinstitute.org \\
       Center for Computational Mathematics\\
      Flatiron Institute, New York 
      \and Fabian Schaipp \\   fabian.schaipp@tum.de \\
       Department of Mathematics\\
      Technical University of Munich}

\begin{document}

\maketitle

\newcommand{\FUVAL}{\texttt{FUVAL}}

\begin{abstract}
Here we develop variants of SGD (stochastic gradient descent) with an adaptive step size that make use of the sampled loss values.  In particular, we focus on solving a finite sum-of-terms problem, also known as empirical risk minimization.  We first detail an idealized adaptive method called \texttt{SPS}$_+$ that makes use of the sampled loss values and assumes knowledge of the sampled loss at optimality.  This \texttt{SPS}$_+$ is a minor modification of the SPS (Stochastic Polyak Stepsize) method,  where the step size is enforced to be positive.  We then show that \texttt{SPS}$_+$ achieves the best known rates of convergence for SGD in the Lipschitz non-smooth. We then move onto to develop \FUVAL{}, a variant of \texttt{SPS}$_+$ where the loss values at optimality are gradually learned, as opposed to being given.  We give three viewpoints of \FUVAL{},  as a projection based method, as a variant of the prox-linear method, and then as a particular online SGD method. We then present a convergence analysis of \FUVAL{} and experimental results. The shortcomings of our work is that the convergence analysis of \FUVAL{} shows no advantage over SGD. Another shortcomming is that currently only the full batch version of \FUVAL{} shows a minor advantages of GD (Gradient Descent) in terms of sensitivity to the step size. The stochastic version shows no clear advantage over SGD. We conjecture that large mini-batches are required to make \FUVAL{} competitive.

Currently the new \FUVAL{} method studied in this paper does not offer any clear theoretical or practical advantage.
We have chosen to make this draft available online nonetheless because of some of the analysis techniques we use, such as the non-smooth analysis of \texttt{SPS}$_+$, and also to show  an apparently interesting approach that currently does not work.
\end{abstract}


 \section{Introduction}

Consider the empirical risk minimization problem
\begin{equation} \label{eq:main}
w^* \in \underset{{w \in \R^d}}{\argmin} \; f(w), \quad  f(w):=\frac{1}{n}\sum_{i=1}^n f_i(w),
\end{equation}
where we assume  that $f$ is bounded below,  continuously differentiable, and the set of minimizers  $\cW^* \subset \R^d$ is nonempty. We denote the optimal value of~\eqref{eq:main} by $f_* := f(w^*) \in \R$. Let $w^0$ be a given initial point.

%

Here we consider iterative stochastic gradient methods that also make use of the loss values $f_i(w^t)$.
Despite the loss value $f_i(w^t)$ being key for monitoring the progress of methods for solving~\eqref{eq:main},  these values are seldom used in the updates of stochastic methods.  Notable exceptions are line search procedures,  and  \texttt{SGD} with a Polyak stepsize~\cite{SPS} whose iterates are given by
\begin{equation}\label{eq:spsloizou}
w^{t+1} \; = \; w^t - \frac{f_i(w^t) - \inf f_i}{\norm{\nabla f_i(w^t)}^2} \nabla f_i(w^t),
\end{equation}
where $i$ is sampling i.i.d and uniformly from $\{1, \ldots, n\}$ at each iteration. 
The issue with~\eqref{eq:spsloizou} is that requires knowing $\inf f_i$.
Another closely related method,  given in~\cite{MOTAPS}, is 
\begin{equation}\label{eq:sps}
w^{t+1} \; = \; w^t - \frac{f_i(w^t) -  f_i(w^*)}{\norm{\nabla f_i(w^t)}^2} \nabla f_i(w^t).
\end{equation}
We will call~\eqref{eq:sps}  SPS method, though this acronym is also used in other work to refer to~\eqref{eq:spsloizou}.
The method~\eqref{eq:sps} has now two issues:  Again the $ f_i(w^*)$ are often not known,  and the resulting step  size $\frac{f_i(w^t) -  f_i(w^*)}{\norm{\nabla f_i(w^t)}^2} $ may be negative.

Our objective here to develop methods that, like~\eqref{eq:spsloizou} and~\eqref{eq:sps}, make use of the loss values, but unlike these methods, does not require knowing  $\inf f_i$ or  $f_i(w^*)$.

\subsection{Function Splitting and Variable Splitting}

We will design our method by using projections onto constraints. To do this, we first need to re-write~\eqref{eq:main} in such a way that each data point (or batch) is split across constraints. 
One way to do this is to use a  \emph{variable splitting} trick which creates duplicates of the variables $x_i \in \R^d$ for $i=1,\ldots, n$ as follows
 \begin{align}
 \min_{w, x_1,\ldots, x_n}\frac{1}{n}&\sum_{i=1}^n f_i(x_i) \\
s.t. \quad w &=  x_i, \quad \mbox{for }i=1,\ldots, n.  \label{eq:varsplittau1}
 \end{align}
By creating a copy of the variables for each $f_i$, and thus for each data point, we can use a coordinate descent method or ADMM to arrive at an incremental method.  This approach is well  suited for the distributed setting~\cite{BoydPCPE11} where each $x_i$ could be stored on an $i$th distributed compute node.

Here we take a different approach and use a  function splitting trick, which  introduces $n$ slack variables $s_i \in \R$ for $i=1,\ldots n,$ and splits the loss function across multiple rows\footnote{Row here refers to the viewpoint that the data is represented as a matrix of shape $n_{\text{samples}}\times n_{\text{features}}$.} as follows
\begin{align}
\min_{w\in \R^d, s \in \R^{n}}  \frac{1}{n}\sum_{i=1}^n s_i& \; \nonumber \\
s.t.\quad  f_i(w) & \;\leq \; s_i, \label{eq:functionsplit_vanilla}
\end{align}
where each
$s_i$ is the \emph{target loss} for the $i$th data point.   The solution to~\eqref{eq:functionsplit_vanilla} is equivalent to that of~\eqref{eq:main},  since at optimality the inequality constraints must be satisfied with equality.
By splitting the function, we have also split the \emph{data} across rows since each $f_i(w)$ depends on a separate $i$-th data point (or batch). This simple fact 
allows for the design of incremental methods for solving~\eqref{eq:functionsplit_vanilla} based on subsampling.  Furthermore,  if the $f_i$ functions are convex,  then~\eqref{eq:functionsplit_vanilla} is a convex program.

 \subsection{Background}

This work follows a line of work on the Stochastic Polyak stepsize which was re-ignited with~\cite{SPS,ALI-G,SGDstruct}. 
Earlier work on the Polyak step size in the deterministic setting started with Polyak himself~\cite{polyak1987introduction}. In~\cite{Hazanpolyak} the authors also developed a method for learning the optimal loss value on the fly for the deterministic Polyak step size method.

Recent work on the stochastic Polyak step size include~\cite{MOTAPS},  which shows how to make use of the optimal total loss value $f(w^*)$ within the stochastic setting.  In~\cite{MOTAPS} that authors also developed their method through a projection viewpoint and as an online SGD method, both of which we leverage here.  

Furthermore~ variants of the stochastic Polyak step size are connected to model based methods~\cite{ALI-G,Chada-accel-model-2021,asi2019importance} and to bundle methods~\cite{Borat}. We also develop a model based viewpoint of our \FUVAL{} method, and use the theory developed in~\cite{Davis2019} to analyse our method in the non-smooth setting.

Our approach is also closely related to~\cite{slackpolyak2022}, where the authors solve the \emph{approximate interpolation} equations
\begin{align}\label{eq:minmax}
    \min s \quad \mbox{subject to } f_i(w) \leq s, \quad \mbox{for }i=1,\ldots, n.
\end{align}
This objective~\eqref{eq:minmax} is apparently similar to ours~\eqref{eq:functionsplit_vanilla}, but the fundamental difference is that our objective is a reformulation of~\eqref{eq:main} whereas \eqref{eq:minmax} is only approximately equivalent to solving~\eqref{eq:main} under a so called $\epsilon$--approximate interpolation condition.

Other techniques for developing an adaptive stepsize include a stochastic line search~\cite{vaswani2019painless}, using local smoothness estimates to create an adaptive scheduling~\cite{pmlr-v119-malitsky20a},  coin tossing techniques~\cite{orabona2019modern},  and variants of AdaGrad~\cite{ADAGRAD},  which arguably include  the notorious Adam method~\cite{ADAM}. But here we will not discus these approaches, and consider them orthogonal techniques.  

\section{Convergence knowing  the $f_i(w^*)$'s}

Before moving on to analysing our new method, we first analyse a variant of the \texttt{SPS} method that requires knowing the $f_i(x^*)$'s.  Let the  \texttt{SPS}$_+$ method be given by
\begin{equation}\label{eq:spspos}
w^{t+1} \; = \; w^t - \frac{(f_{i}(w^t) - f_{i}(w^*))_+}{\norm{\nabla f_{i}(w^t)}^2} \nabla f_{i}(w^t),
\end{equation}
where $i\in[n]$ is sampled uniformly at random and we denote $a_+ = \max \{a,0 \}.$
The only difference between  \texttt{SPS}$_+$ and \texttt{SPS} in~\eqref{eq:sps} is that we have taken the positive part of $f_{i}(w^t) - f_{i}(w^*)$. 
This   \texttt{SPS}$_+$ variant can be motivated using an upper bound derived from star convexity, or as a particular projection method,  as we show next.

 \subsection{Star convex viewpoint}

Consider the iterates of SGD given by
\[w^{t+1} =w^t - \gamma_t \nabla f_{i}(w^t), \]
where $\gamma_t>0$ are positive learning rates that we need to choose.
We will now choose $\gamma_t$ that gives the best one step progress towards the solution for star-convex functions.
\begin{assumption}[Star-convex functions] Let $f_i$ be such that for all $w\in \R^d$ and all $w^* \in \mathcal{W}^*$
\begin{equation}\label{eq:stari}
f_i(w^*) \geq f_i(w) + \dotprod{\nabla f_{i}(w), w^*-w }, \quad \mbox{for }i=1,\ldots, n.
\end{equation}
\end{assumption}

That is,  expanding the squares we have that
\begin{align}\label{eq:iterateexpsgd}
\norm{w^{t+1} -w^*}^2  = \norm{w^t -w^*}^2 - 2\dotprod{\gamma_t  \nabla f_i(w^t), w^t-w^* } + \norm{\gamma_t  \nabla f_i(w^t)}^2.
\end{align}
Using star-convexity~\eqref{eq:stari} and that $\gamma_t >0$,  we have that
\begin{align}\label{eq:polyakmot}
\norm{w^{t+1} -w^*}^2  & \overset{\eqref{eq:stari}}{\leq} \norm{w^t -w^*}^2 - 2\gamma_t(f_i(w^t) -f_i(w^*)) + (\gamma_t)^2 \norm{  \nabla f_i(w^t)}^2.
 \end{align}
 We now determine $\gamma_t$ by minimizing the right hand side of the above.
 \begin{lemma}
 The step size $\gamma\geq 0$ that minimizes the right hand sides of~\eqref{eq:polyakmot} is given by 
 \begin{eqnarray}\label{eq:SPSplusstep}
  \gamma_t = \frac{(f_i(w^t) -f_i(w^*))_+}{\norm{  \nabla f_i(w^t)}^2}.
 \end{eqnarray}
 \end{lemma}
 \begin{proof}
To solve
 \begin{equation}
    \gamma_t \;=\; \underset{\gamma \geq 0}{\argmin} \; q(\lambda) \;:=\; -2\gamma(f_i(w^t) -f_i(w^*)) + \gamma^2 \norm{  \nabla f_i(w^t)}^2,
 \end{equation}
 we first take the derivative in $\gamma$ and compute the solution \emph{without} the positivity constraint, which gives
 \begin{eqnarray}\label{eq:lohoz9hz4s}
 \hat \gamma = \frac{f_i(w^t) -f(w^*)}{\norm{  \nabla f_i(w^t)}^2}.
 \end{eqnarray}
  Since this is the unconstrained solution, we have that $q(\hat \gamma) \leq \min_{\gamma \geq 0} q(\gamma).$ Thus it is the solution so long as it does not violate that positivity constraint, that is so long as $f_i(w^t) -f(w^*) \geq 0.$ Alternatively, the other candidate solution is given by $\gamma=0$  which is a KKT point with active constraint. Putting these two alternatives together, we have that the solution is given by~\eqref{eq:SPSplusstep}.
 
 \end{proof}
 
This derivation of \texttt{SPS}$_+$ is almost identical to the derivation of SPS given in~\cite{polyak1987introduction,SPS,ALI-G}. The only difference is that we have explicitly used the positive constraint $\gamma\geq 0$.

 \subsection{Projection viewpoint}
 
We can also derive \texttt{SPS}$_+$ as a projection method for solving nonlinear \emph{inequalites}. The content of this subsection is not particularly novel (cf.\ \cite{SGDstruct}) but we repeat it here in order to motivate our method. 
Indeed,  first note that we can re-write our empirical risk problem~\eqref{eq:main} as the following system of nonlinear inequalities
\begin{equation}\label{eq:inequalites}
\mbox{Find $w\in \R^d$ such that: } \qquad f_i(w) \; \leq \; f_i(w^*), \quad \mbox{for }i=1,\ldots, n.
\end{equation}
To see the equivalence between~\eqref{eq:inequalites} and~\eqref{eq:main} first note that any solution $w$ to~\eqref{eq:inequalites} must be one where all the constraints are saturated such that $f_i(w) = f_i(w^*)$ for  $i=1,\ldots, n.$ Otherwise if a single constraint was not saturated with say $f_1(w)< f_1(w^*)$ then we would have $\frac{1}{n} \sum_{i=1}^n f_i(w) < f(w^*),$ which is not possible by definition of $w^*.$

We can now focus on solving~\eqref{eq:inequalites},  for which we devise an iterative projection method.  At each iteration,  we first sample 
 $j \in \{1, \ldots, n\}$ i.i.d and the corresponding $j$th constraint $f_j(w) \leq f_j(w^*).$ We then try to take one step towards satisfying this constraint.  Since this is still a potentially difficult nonlinear constraint,  we linearize this constraint, and then project our previous iterate onto this linearization, that is 
\begin{align}\label{eq:projsps}
\begin{split}
w^{t+1} = & \underset{w \in \R^d }{\rm{argmin}} \norm{w-w^t}^2 \\
& \mbox{ subject to } f_j(w^t) + \dotprod{\nabla f_j(w^t), w-w^t} \; \leq\; f_j(w^*). 
\end{split}
\end{align} 
 The solution to~\eqref{eq:projsps} is given by the \texttt{SPS}$_+$ update~\eqref{eq:spspos},  which follows by applying Lemma~\ref{lem:L2ineqconst}.
 
 This viewpoint is closely related to the viewpoint of SPS as a Newton--Raphson method~\cite{ALI-G,MOTAPS}, where $\leq$ in \eqref{eq:projsps} is replaced by an equality constraint.  Through this viewpoint,  \texttt{SPS}$_+$ can also been seen as an extension of Motzkin's method~\cite{Motzkin} for linear feasibility.

\subsection{Convergence analysis}
Next we analyse \texttt{SPS}$_+$ and show that it achieves the best possible rate for any adaptive SGD method for Lipschitz and convex functions.
Since the proof does not require that the objective function be a finite sum,  we give the statement and proof for minimizing a general expectation.

\begin{restatable}[Convergence of \texttt{SPS}$_+$]{theorem}{theosps}
\label{theo:sps}
Consider the problem of solving
$$w^{*} \in \underset{w\in\R^d}{\argmin} \;\;\EE{x\sim \cD }{f_{x}(w)}$$ where $x$ is sampled data and $\cD$ is a an unknown data distribution. 
Let $w^t$ be the iterates of \texttt{SPS}$_{+}$ where
\begin{equation}
w^{t+1} \; = \; w^t - \frac{(f_{x}(w^t) - f_{x}(w^*))_+}{\norm{\nabla f_{x}(w^t)}^2} \nabla f_{x}(w^t),
\end{equation}
where $x\sim \cD$ is sampled i.i.d at each iteration. 
 If  $f_x$ is star-convex around $w^*$ then the iterates are Fej\'er monotonic
\begin{eqnarray}
    \norm{w^{t+1} -w^*}^2 \; \leq \; \norm{w^t -w^*}^2 -  \frac{(f_x(w^t) - f_x(w^*))_+^2}{\norm{\nabla f_x(w^t)}^2} .
\end{eqnarray}

 Furthermore, we have that
 \begin{enumerate}
     \item If $f_x$ is $G$--Lipschitz then
       \begin{align}\label{eq:SPSposconvlip}
       \min_{t=1, \ldots,T} \E{f(w^t) - f(w^*)}& \leq \frac{G}{\sqrt{T}}\norm{w^0 -w^*}.
 \end{align}
 \item If $f_x$ is $L$--smooth  and the interpolation condition given by
 \begin{eqnarray} \label{eq:interpolation}
 f_x(w^*) &=& \inf_w f_x(w),  \quad \mbox{for every }x \in \mbox{supp}(\cD)
 \end{eqnarray}
 holds then
   \begin{align}\label{eq:SPSposconvsmooth}
      \min_{t=0, \ldots, T-1}\E{ f(w^t) - f(w^*)}
 & \leq \frac{2L}{T}\norm{w^0 -w^*}^2.
 \end{align}
    
 \end{enumerate}
\end{restatable}
\begin{proof}
Expanding the squares and using star-convexity we have that
\[
\norm{w^{t+1} -w^*}^2  \; \overset{\eqref{eq:stari}}{\leq} \;\norm{w^t -w^*}^2 - 2\gamma_t(f_x(w^t) -f_x(w^*)) + (\gamma_t)^2 \norm{  \nabla f_x(w^t)}^2,\]
where $\gamma_t =  \frac{(f_{x}(w^t) - f_{x}(w^*))_+}{\norm{\nabla f_{x}(w^t)}^2} $ is the \texttt{SPS}$_{+}$ step size.
Substituting in $\gamma^t $ gives
\begin{align}\label{eq:polyakmot2}
\norm{w^{t+1} -w^*}^2  & \leq \norm{w^t -w^*}^2 - 2\frac{(f_x(w^t) - f_x(w^*))_+}{\norm{\nabla f_x(w^t)}^2} (f_x(w^t) -f(w^*)) + \frac{(f_x(w^t) - f_x(w^*))_+^2}{\norm{\nabla f_x(w^t)}^2}.
 \end{align}
If $f_x(w^t) - f_x(w^*) >0 $ then  $(f_x(w^t) - f_x(w^*))_+ = f_x(w^t) - f_x(w^*)$,  thus re-arranging
   \begin{align}\label{eq:polyakmot22}
\frac{(f_x(w^t) - f_x(w^*))_+^2}{\norm{\nabla f_x(w^t)}^2}  & \leq \norm{w^t -w^*}^2 -\norm{w^{t+1} -w^*}^2.
 \end{align}
 Alternatively,  $f_x(w^t) - f_x(w^*) \leq 0 $ then $(f_x(w^t ) - f_x(w^*))_+ =0$ and consequently
   \begin{align}\label{eq:polyakmot23}
\frac{(f_x(w^t) - f_x(w^*))_+^2}{\norm{\nabla f_x(w^t)}^2}  =  0 & \leq \norm{w^t -w^*}^2 - \norm{w^{t+1} -w^*}^2 .
 \end{align}
Thus in either case we have that
\begin{eqnarray}
    \norm{w^{t+1} -w^*}^2 \; \leq \; \norm{w^t -w^*}^2 -\frac{(f_x(w^t) - f_x(w^*))_+^2}{\norm{\nabla f_x(w^t)}^2}.
\end{eqnarray}
Dividing by $T$ and summing up for $t=0, \ldots, T-1$ on both sides and using telescopic cancellation gives
  \begin{align}\label{eq:polyakmot22}
\frac{1}{T}\sum_{t=0}^{T-1}\frac{(f_x(w^t) - f_x(w^*))_+^2}{\norm{\nabla f_x(w^t)}^2}  & \leq \frac{1}{T}\norm{w^0 -w^*}^2.
 \end{align}
 
 Now we consider one of the following assumptions.
 \begin{enumerate}
     \item 
If $f_i$ is $G$--Lipschitz, then $\norm{\nabla f_x(w)}^2 \leq G^2$ and consequently
  \begin{align}\label{eq:polyakmot224}
\frac{1}{T}\sum_{t=0}^{T-1}\frac{(f_x(w^t) - f_x(w^*))_+^2}{G^2}  & \leq \frac{1}{T}\sum_{t=0}^{T-1}\frac{(f_x(w^t) - f_x(w^*))_+^2}{\norm{\nabla f_x(w^t)}^2}  \;  \leq \frac{1}{T}\norm{w^0 -w^*}^2.
 \end{align}
 Multiplying through by $G^2$ and taking expectation gives
   \begin{align}\label{eq:polyakmot2254}
\frac{1}{T}\sum_{t=0}^{T-1}\E{(f_x(w^t) - f_x(w^*))_+^2}  & \leq  \frac{G^2}{T}\norm{w^0 -w^*}^2.
 \end{align}
 Using Jensen's, and that $x \mapsto x_+^2$ is a convex function we have that
 $ \mathbb{E}[f(x)] \geq f(\mathbb{E}[X])$
 \[(f(w^t) - f(w^*))_+^2 \; = \; (\EE{t}{f_x(w^t) - f_x(w^*)})_+^2  \; \leq \;\EE{t}{(f_x(w^t) - f_x(w^*))_+^2} \]
 This combined with~\eqref{eq:polyakmot2254} gives
    \begin{align}\label{eq:polyakmot22544}
\min_{t=1, \ldots,T} \E{(f(w^t) - f(w^*))_+^2} \; \leq  \frac{1}{T}\sum_{t=0}^{T-1}\E{(f(w^t) - f(w^*))_+^2}  & \leq  \frac{G^2}{T}\norm{w^0 -w^*}^2.
 \end{align}
 We now can drop the positive part since $f(w^t) \leq f(w^*)$ by definition. 

 Again, using the convexity of $x \mapsto x^2$ and Jensen's, we have that 
     \begin{align}\label{eq:polyakmot225445}
\min_{t=1, \ldots,T} \E{f(w^t) - f(w^*)}^2 \; \leq  \min_{t=1, \ldots,T} \E{(f(w^t) - f(w^*))^2} & \leq  \frac{G^2}{T}\norm{w^0 -w^*}^2.
 \end{align}

 Taking the square root and using that
 \[ \sqrt{\min_{t=1, \ldots,T} \E{f(w^t) - f(w^*)}^2} = \min_{t=1, \ldots,T} \sqrt{\E{f(w^t) - f(w^*)}^2}= \min_{t=1, \ldots,T} \E{f(w^t) - f(w^*)}\]
 gives the result.
 \item 
 This result almost follows from by the Theorem 4.4 in~\cite{SGDstruct}, but there is a subtle additional assumption on positivity of the loss, and Theorem 3.4~\cite{SPS} where the method is slightly different. So we provide the proof here for completion.
 
If $f_x$ is convex and $L$--smooth, then 
 by~Lemma~\ref{lem:convsmoothinter} we have that
\begin{equation}
 (f_x(w) - \inf f_x)_+ \geq   f_x(w) - \inf f_x \geq \frac{1}{2L} \norm{\nabla f_x (w)}^2 \enspace.
\end{equation}
Using the above in~\eqref{eq:polyakmot22} gives
  \begin{align}\label{eq:tempnioznioen}
      \min_{t=0, \ldots, T-1}\frac{\E{ (f_x(w^t) - f_x(w^*))_+}^2}{\E{f(w^t) -\inf f_x} }  \; \leq \;
\min_{t=0, \ldots, T-1}\E{ \frac{(f_x(w^t) - f_x(w^*))_+^2}{f_x(w^t) -\inf f_x }} & \leq \frac{2L}{T}\norm{w^0 -w^*}^2.
 \end{align}
 Finally using the interpolation condition~\eqref{eq:interpolation} gives the result.


 \end{enumerate}

\end{proof}
The result~\eqref{eq:SPSposconvlip} for Lipschitz functions in the deterministic setting was already known,  see for instance Theorem 8.17 in~\cite{Beck2017}.  This same result has also been proven in Theorem C.1 in~\cite{SPS},  but under the additional assumption that interpolation holds. Under interpolation,  we have that $f_i(w^*) = \inf f_i,$ and consequently one can  use of the step size $\frac{f_{i}(w^t) -\inf f_{i}}{\norm{\nabla f_{i}(w^t)}^2} =  \frac{f_{i}(w^t) - f_{i}(w^*)}{\norm{\nabla f_{i}(w^t)}^2}$ since it is positive. The novelty in our proof is that we take the positive part of the step size, and are thus able to prove the same result without the additional interpolation assumption.

The rate of convergence in the Lipschitz case~\eqref{eq:SPSposconvlip} is only off by $\frac{1}{2}$ as compared to optimal rate for any online stochastic gradient method which is $\frac{G \norm{w^0-w^*}}{2\sqrt{T}}$,  see Theorem 5.1 in~\cite{orabona2019modern}. Furthermore, 
to achieve the same rate in~\eqref{eq:SPSposconvlip} with SGD under the same assumptions,  one needs to know a radial distance $D >0$ such that $w^* \in \{ w \;: \; \norm{w-w^0} \leq D\} $.  With knowledge of this distance $D$, each step of SGD is then interlaced with a projection onto the ball  $ \{ w \;: \; \norm{w-w^0} \leq D\} $.   The \texttt{SPS}$_+$ method needs to know the values $f_i(w^*)$ (or $f_x(w^*)$ in the proper stochastic setting),  but not the diameter $D.$


As for the smooth setting under interpolation in~\eqref{eq:SPSposconvsmooth},  this matches the best rate of SGD for when we know the smoothness constant $L$.~\cite{vaswani2018fast,SGDstruct}.

Though Theorem~\ref{theo:sps} achieves the best rates in each setting,  estimating the $f_i(w^*)$ values can be difficult.
The main question of this paper is \emph{Can we learn the  $f_i(w^*)$ values on the fly?}.  This is what we attempt in the remaining of the paper.


On a bibliographic note,  the smooth convergence rate~\eqref{eq:SPSposconvsmooth}, was already given in Theorem 4.4 in~\cite{SGDstruct} which analyses \texttt{SPS}. We give the result again because we analyse \texttt{SPS}$_+$ instead, which required some very minor changes to the proof.

\section{\FUVAL{}: An adaptive Function Value Learning method}

We now let go of having to know the $f_i(x^*)$'s and use the function splitting re-formluation~\eqref{eq:functionsplit_vanilla} to develop our new method \FUVAL{}, stated in Algorithm \ref{alg:fuval}.  We give three different viewpoints of our method based on \emph{projections},  \emph{prox-linear} and the \emph{online SGD} methods. As such, we actually present three different methods that are very similar and united in Algorithm \ref{alg:fuval}. Each viewpoint will reveal a natural motivation for the choice of some of the hyperparameters in Algorithm \ref{alg:fuval}. While \FUVAL{} seemingly needs the four parameters $\delta_t,~\lambda_t,~c,~\gamma$, we will explain that a natural choice is $\gamma=1,~ c\in \{1, + \infty\}$ and that $\delta_t$ and $\lambda_t$ can be parameterized jointly using only one single hyperparameter (cf.\ Remark \ref{rem:scale}).

\begin{algorithm}
\caption{\FUVAL{}}
\label{alg:fuval}
\begin{algorithmic}[1]
\State {\bf Inputs:}  step sizes $\lambda_t >0,$ $\delta_t >0$, parameter $\gamma \in [0,1]$, penalty multiplier $c \geq 1$.
\State {\bf Initialize:} $w^0 \in\mathbb{R}^d$ and $s_i^0 \in \mathbb{R}$ for $i=1,\ldots, n.$
\For{$t =0,\ldots, T$} 
\State Sample $j_t$ randomly from $[n]$.
\State $\tau_t = \min\Big\{c, \frac{\big(f_{j_t}(w^t) - s_{j_t}^t + \delta_t\big)_+}{\delta + \lambda \|\nabla f_{j_t}(w^t)\|^2 } \Big\}$
\State $\displaystyle    w^{t+1} \;= w^t - \gamma \tau_t \lambda_t \nabla f_{j_t}(w^t) $
\State $\displaystyle   
    s^{t+1}_i \;= \begin{cases}
      s_i^t - \gamma \delta_t(\tau_t-1), \quad &\text{if } i = j_t, \\
      s_i^t, \quad &\text{else.}
    \end{cases}
$
\EndFor
\State {\bf Output:} $w^{T+1}, s^{T+1}$
\end{algorithmic}
\end{algorithm}

\subsection{Projection viewpoint}

By leveraging the equivalence between~\eqref{eq:functionsplit_vanilla} and our original sum-of-terms problem~\eqref{eq:main}, we can now solve ~\eqref{eq:functionsplit_vanilla} using an incremental projection method. This projection method is analogous to the projection viewpoint of \texttt{SPS}$_+$ given in~\eqref{eq:projsps}.

To develop an incremental method for solving~\eqref{eq:functionsplit_vanilla},    
 at each iteration we sample $j \in \{1, \ldots, n\}$ i.i.d and the $j$-th constraint.  We then linearize the constraint and project our current iterate $(w^t,s^t)$ onto this constraint as follows
\begin{align}\label{eq:funcvallearn}
\begin{split}
w^{t+1}, s^{t+1} = & \underset{w \in \R^d,  s \in \R^n}{\rm{argmin}}~ s_j + \frac{1}{2\lambda} \norm{w- w^t}^2+\frac{1}{2\delta } \Vert s- s^t \Vert^2 \\
&\mbox{ subject to }f_j(w^t)+\dotprod{\nabla f_j(w^t),w -w^t } \leq s_j,
\end{split}
\end{align}
where $\delta>0$ and $\lambda>0$ are parameters.  We need both of these two parameters to match up \emph{units} in the objective,  as we detail in the following remark.

\begin{remark}[Scale Invariance]\label{rem:scale}
Our goal is to have a method that is scale invariant,  in the sense that it should work equally well for minimizing $f(w)$ or any scaled version such as $100\cdot f(5\cdot w)$.  To achieve this, 
we need to choose $(\lambda,\delta)$ so that~\eqref{eq:funcvallearn} has the same units as $f_j$.
Thus $s_j$ must have the same units as $f$ due to the constraint.  Consequently, and informally, 
by choosing units$(\delta) = \mbox{units}(f)$ and units$(\lambda) = \mbox{units}(w)^2/\mbox{units}(f)$ we have that the units match the objective.   
For example,  let $f_i$ be $G_i$--Lipschitz and let $G_{\max} = \max_{i=1,\ldots, n} G_i.$
To match units,  we could set
\begin{eqnarray}
\lambda = \eta \, \frac{\norm{w^0}^2}{G_{\max} } \quad \mbox{and}\quad  \delta = \eta \, G_{\max}, 
\end{eqnarray}
where $\eta >0$ is our \emph{dimensionless} tunable parameter. If estimates of $G_{\max}$ are not available then
\begin{eqnarray}
\lambda = \eta \, \frac{\norm{w^0}^2}{f(w^0)} \quad \mbox{and}\quad  \delta =\eta \, f(w^0).
\end{eqnarray}
We use this insight to set default choices for $\delta$ and $\lambda$ later on in Section~\ref{sec:numerics}
\end{remark}
The projection~\eqref{eq:funcvallearn} also has a convenient 
 solution.
\begin{restatable}[Projection Update]{lemma}{projupdate}
\label{lem:projupdate}
The solution to~\eqref{eq:funcvallearn} is given by
\begin{align}\label{eq:projupdate}
    \begin{split} 
  \tau_t &= \frac{(f_j(w^t) - s_j^t + \delta)_+}{\delta+\lambda \Vert \nabla f_j(w^t) \Vert^2},  \\
    w^{t+1} &= w^t - \lambda \tau_t \nabla f_j(w^t),   \\
    s_j^{t+1} &= s_j^t - \delta + \delta\tau_t,    \\
    s_i^{t+1} &= s_i^t \text{ for } i \neq j.
    \end{split}
\end{align}
\end{restatable}
Clearly, \eqref{eq:projupdate} is \FUVAL{} with $\lambda_t=\lambda$, $\delta_t=\delta$, $\gamma=1$, $c=\infty$.


\subsection{Prox-Linear viewpoint}
%

Here we provide another viewpoint of our method as a variant of the prox-linear  method \cite{Davis2019, Drusvyatskiy2019}. This viewpoint is based on solving the  $\ell_1$-penalty reformulation of~\eqref{eq:functionsplit_vanilla} given by
%
\begin{align}
    \label{prob:pos-part}
    \min_{w \in \mathbb{R}^d,s\in \mathbb{R}^n} g(w,s), \quad g(w,s) :=\frac{1}{n}\sum_{i=1}^n \Big(s_i + c\, (f_i(w) - s_i)_+ \Big) ,
\end{align}
where $c \geq 0$ is the penalty parameter. 
When $c\geq 1$, solving~\eqref{prob:pos-part}
is equivalent to solving~\eqref{eq:functionsplit_vanilla}.

\begin{restatable}[Equivalent Penalty Problem]{lemma}{equivpenalty}
 \label{lem:equivpenalty}
Let $f_i$ be convex for all $i\in[n]$. Let $(w^*,s^*) \in \mathbb{R}^{d+n}$ be a solution to \eqref{prob:pos-part}. Then necessarily $c \geq 1$ and $s_i^*\leq f_i(w^*)$ for all $i\in[n]$. Further, $w^*$ is a global minimum of $f$ and moreover $g(w^*,s^*) = f(w^*)$.
\end{restatable}
\begin{remark}
    The fact that a solution to \eqref{eq:functionsplit_vanilla} is a solution to \eqref{prob:pos-part} if $c\geq 1$ can be seen through the connection to exact penalty functions \cite[Section 17.2]{Nocedal2006}. In fact, $g(w,s)$ is the $\ell_1$-penalty function of problem \eqref{eq:functionsplit_vanilla}.
\end{remark}
Because of the above equivalence between~\eqref{eq:functionsplit_vanilla} and ~\eqref{prob:pos-part},  we focus on solving the penalty problem~\eqref{prob:pos-part}.   One way to minimize~\eqref{prob:pos-part} would be to use SGD (stochastic subgradient descent).  Let $g_i(w,s) := s_i + c (f_i(w) - s_i)_+$.  Thus~\eqref{prob:pos-part} is equivalent to minimizing $\frac{1}{n} \sum_{i=1}^n g_i(w,s).$ To abbreviate let $u = (w,s).$
At each iteration SGD samples a data point $j\in \{1,\ldots, n\}$ and from a given $u^t = (w^t,s^t)$ updates the parameters according to
\begin{equation}\label{eq:sgdpenalty}
u^{t+1} \; =\; \underset{{u= (w,s)\in \R^{d+n}}}{\argmin}  g_j(u^t) +\dotprod{ v, u-u^t} +  \frac{1}{2\lambda_t} \norm{u-u^t}^2 ~~ \text{for} ~~ v\in \partial g_j(u^t),
\end{equation}
where $\lambda_t>0$ are the learning rates. Here, $\partial g_j$ is a suitable subdifferential (e.g.\ the convex subdifferential if $f_j$ is convex).  The closed form solution to~\eqref{eq:sgdpenalty} is the well known SGD update.  The issue with~\eqref{eq:sgdpenalty} is that it approximates $g_j(u)$ by its local linearization,  that is
\[  g_i(u) \approx g_i(u^t) + \dotprod{v, u-u^t}, \quad v\in \partial g_i(u^t).\]

We can build a more accurate approximation, or \emph{model}, of $g_i(u) $ by exploiting the positive term $c (f_i(w) - s_i)_+$.  Indeed, a more accurate approximation of $g_i(u) $ is given by
\begin{equation} \label{eq:modelprox}
 g_i(w,s) \approx  s_{i} + c\Big(f_{i}(w^t)+ \langle \nabla f_{i}(w^t), w-w^t \rangle - s_{i}\Big)_+,
\end{equation} 
where we linearized the term within the positive part,  as opposed to linearizing $ g_i(w,s) $ as was done in the SGD method.  Using the better approximation~\eqref{eq:modelprox} together with a proximal update,  gives the following update 
\begin{align}\label{eqn:two-scale-update}
    w^{t+1}, s^{t+1} = \argmin_{w, s} s_{j_t} + c\Big(f_{j_t}(w^t)+ \langle \nabla f_{j_t}(w^t), w-w^t \rangle - s_{j_t}\Big)_+ + \frac{1}{2\lambda_t}\|w-w^t\|^2 + \frac{1}{2\delta_t}\|s-s^t\|^2,
\end{align}
where $\lambda_t>0$ and $\delta_t>0$ are tunable parameters.  This update~\eqref{eqn:two-scale-update} is a variant of the prox-linear method as we detail in Section~\ref{sec:model-prox}.  
 Indeed~\eqref{eqn:two-scale-update} can be seen as a proximal method where the proximal operator is computed with respect to the metric induced by the diagonal matrix
\begin{equation}\label{D:metric}
    \mD := 
    \begin{pmatrix}
    \frac{1}{\lambda} \mI_d & \zeros_{d,n} \\
    \zeros_{n,d} & \frac{1}{\delta} \mI_n
    \end{pmatrix}.
\end{equation}
Fortunately~\eqref{eqn:two-scale-update} has a closed form solution,  which we give in the following lemma.

\begin{restatable}[Prox-Linear Update]{lemma}{modelupdate}
 \label{lem:modelupdate}
The closed form solution to~\eqref{eqn:two-scale-update}
is given by
\begin{align}\label{eq:prox-lin-method}
\begin{split}
    \tau_t &:= \min\Big\{ c, \frac{\big(f_{j_t}(w^t) - s_{j_t}^t + \delta_t\big)_+}{\delta_t +\lambda_t\|\nabla f_{j_t}(w^t)\|^2} \Big\}, \\
    w^{t+1} &= w^t - \tau_t  \lambda_t\nabla f_{j_t}(w^t), \\
     s^{t+1}_j &=  s_j^t - \delta_t + \tau_t \delta_t, \quad \text{if } j = j_t,  \\
     s^{t+1}_j &=  s_j^t, \quad \text{if } j \neq j_t, 
\end{split}
\end{align}
\end{restatable}
Clearly, ~\eqref{eq:prox-lin-method} is \FUVAL{} with $\gamma=1$.
The difference between~\eqref{eqn:two-scale-update} and the  standard prox-linear method is that we have two tunable parameters $\lambda_t>0$ and $\delta_t>0$,  instead of just one parameter where $\lambda_t = \delta_t$ in the standard prox-linear method.  We introduce two parameters so that we can arrive at a scale-invariant method,  see Remark~\ref{rem:scale}.  

Using the connection to model-based methods,  we adapt the convergence theory provided by~\cite{Davis2019} to arrive at the following Corollary. This Corollary also follows closely the proof of Theorem 5.2 in~\cite{meng2023modelbased}.
\begin{restatable}[Prox-Linear Convergence]{corollary}{fmodelbasedconvhat}
 \label{cor:fmodelbasedconvhat}
Let $f_i$ be  convex and $G_i$-Lipschitz for all $i\in[n]$.
If $c\geq1$ for all $i\in[n]$, $\lambda_t=\frac{\lambda}{\sqrt{t+1}}$ and $\delta_t=\frac{\delta}{\sqrt{t+1}}$, then, for $T\in \N$, the iterates~\eqref{eq:prox-lin-method} satisfy
\begin{align} \label{eqn:estimate-convex-hat}
        \mathbb{E}\Big[f(\bar{w}^{T}) - f(w^*)\Big] \leq \frac{\tfrac{1}{\lambda}\|w^0-w^*\|^2 +\frac{1}{\delta}\|{s}^0-{s}^*\|^2}{4 (\sqrt{T+2}-1) } + \frac{1}{n}\sum_{i=1}^n \Big(1+  \sqrt{\frac{\lambda}{\delta}G_i^2 + 1}\Big)\frac{\delta(1+\ln(T+1))}{2\sqrt{T+2}-2},
\end{align}
where $\bar{w}^{T} := \frac{1}{T+1}\sum_{t=0}^T \lambda_t w^{t+1}$.
\end{restatable}

\subsection{An online SGD viewpoint}

Our final viewpoint of our method is as a type of online SGD method.  We use this viewpoint to establish convergence of our method for smooth and convex functions.  But first,  we need to introduce a relaxation step into our method.


 By relaxation, we mean that, instead of doing the update
\[w^{t+1} = w^t + d^t\]
where $d^t$ is the update vector, we shrink the size of the update using a  relaxation parameter   $\gamma \in \; ] 0, \;1]$  and update according to
\[w^{t+1} = w^t + \gamma d^t.\]
Applying the relaxation step in both variables $w$ and $s$ we have in~\eqref{eq:projupdate} we arrive at
\begin{equation}\label{alg:relaxedversion}
    \begin{cases}
  \tau_t & = \; \displaystyle    \frac{(f_{j_t}(w^t)  -s_{j_t}^t+ \delta)_+}{\delta+\lambda \Vert \nabla f_{j_t}(w^t) \Vert^2}\\
 w^{t+1} &= \;w^t -\gamma \lambda \tau_t \nabla f_{j_t}(w^t) \\
  s_{j_t}^{t+1} &=\; s_{j_t}^t + \gamma \delta( \tau_t -1)\\
 s_i^{t+1} &=\; s_i^t \quad \text{ for all } i \neq j_t, \\
    \end{cases}
\end{equation}
where $j_t$ is sampled i.i.d at each iteration from $\{1,\ldots, n\}.$
Clearly, ~\eqref{alg:relaxedversion} is \FUVAL{} with $\lambda_t=\lambda$, $\delta_t=\delta$, $c=\infty$.


\subsubsection{Online SGD}

The method~\eqref{alg:relaxedversion} can also be interpreted as an online SGD method applied to minimizing
\begin{equation}\label{eq:L1sgd}
\min_{w \in \R^d, s\in \R^n} \phi_{t}(w,s), \quad \phi_{t}(w,s) :=\frac{1}{n} \sum_{i=1}^n \underbrace{\left(\frac{1}{2} \frac{(f_{i}(w) - s_i+ \delta)_+^2}{\delta +  \lambda\Vert \nabla f_{i}(w^t) \Vert^2} +  s_i    \right)}_{=:\phi_{i,t}(w,s)}  .
\end{equation}
What stands out about~\eqref{eq:L1sgd} is that the objective function now depends on $t$ through $\norm{\nabla f_i(w^t)}^2$ on the denominator.  Despite this dependency on $t$,  we show in the next lemma  solving this online convex problem \eqref{eq:L1sgd} is equivalent to solving our original problem~\eqref{eq:main}.

\begin{restatable}[Equivalent Online Problem]{lemma}{commonminimizers}
\label{L:common minimizers}
Let  $t \in \mathbb{N}$ and let $f$ be convex.
A given $w^* \in \mathbb{R}^d$ is a  minimizer of $f$ if and only if $(w^*,s^*_t)$ is a  minimizer of $\phi_t$, where
 \begin{eqnarray}\label{eq:sstar}
 s_{i,t}^* := (s_t^*)_i = f_i(w^*) - \lambda \Vert \nabla f_i(w^t) \Vert^2, \quad \mbox{for }i=1,\ldots, n,
 \end{eqnarray}
 \end{restatable}
Because~\eqref{eq:L1sgd} depends on $t$,  we need to use an online method to  solve~\eqref{eq:L1sgd}, such as online SGD. 
\begin{restatable}[Online SGD Equivalence]{lemma}{onlineSGDequivalence}
\label{L:online SGD equivalence}
Let $\gamma  \in ]0,1]$.  
The method~\eqref{alg:relaxedversion}  is equivalent to applying the online SGD method to~\eqref{eq:L1sgd} given by
\begin{align}
    w^{t+1} &= w^t - \gamma \lambda \nabla_w \phi_{j_t,t}(w^t,s^t)  \nonumber \\
    s^{t+1} &= s^t - \gamma \delta \nabla_s \phi_{j_t,t}(w^t,s^t).\label{eq:sgdview}
\end{align}
\end{restatable}
Note that the online SGD method in the above lemma applies a different stepsize in the $w$ and $s$ variables.
Another way to see this  is as  an  online SGD method in the metric induced by $\mD$~\eqref{D:metric}, in other words
\[z^{t+1} = z^t - \gamma \mD^{-1} \nabla \phi_{j_t,t}(z^t)\]
where $z^t = (w^t, s^t)$. Will we use this viewpoint in proving convergence.

\subsubsection{Convergence}
We now use this connection to online SGD to provide a convergence theory for smooth and convex functions.
\begin{restatable}[SGD Convergence]{theorem}{cvconvexsmooth}
\label{T:cvconvexsmooth}
Let  $f_i$ be convex and $L_{\max}$--smooth for $i \in [n].$
Let $w^0 \in \mathbb{R}^d$ and $s^0 \in \mathbb{R}^n$ be such that $s_i^0 \geq \inf f_i$.
Let $(w^t,s^t)$ be the sequence generated by the Algorithm \eqref{alg:relaxedversion}, with parameters $\lambda \in ]0,\frac{1}{2 L_{\max}}[$ and $\gamma \in (0,1)$.
Let $T >0$ and let $\bar w^T:=\frac{1}{T}\sum_{t=0}^{T-1} w^t$.
Let $w^* \in {\rm{argmin}}~f$, $s_i^*:= f_i(w^*)$, 
and $\sigma:= \inf f - \frac{1}{n}\sum_{i=1}^n \inf f_i$.
It follows that
\begin{equation*}
    \mathbb{E}[
    f(\bar w^T) - \inf f  
    ]
    \leq
    \frac{\tfrac{1}{\lambda}\norm{w^{0} -w^*}^2+\tfrac{1}{\delta}\norm{s^{0} -s^*}^2}{2\gamma(1-\gamma)(1- \lambda L_{\max} )T}
    +  \sigma \frac{\gamma + \lambda L_{\max}(1-\gamma)}{(1-\gamma)(1- \lambda L_{\max})}
    .
\end{equation*}
\end{restatable}
Theorem~\ref{T:cvconvexsmooth} shows that the method~\eqref{alg:relaxedversion} enjoys a $O(1/T)$ convergence upto a noise radius proportional to $\sigma.$ This is the same rate of convergence for SGD,  see Theorem 4.1 in~\cite{SGDstruct}.  
This result also shows that, like SPS~\cite{SPS} and SGD, our method~\eqref{alg:relaxedversion} also converges at a rate of $O(1/T)$ under \emph{interpolation}. That is, from Lemma 4.15, item 2 in~\cite{handbook2023},  when $\sigma =0$ we have that
\[ 0= f(w^*) - \frac{1}{n} \sum_{i=1}^n \inf f_i = \frac{1}{n} \sum_{i=1}^n (f_i(w^*) - \inf f_i).\]
Since by definition we also have $\inf f_i \leq f_i(w^*)$, the above shows that $w^* \in \argmin_{w\in\R^d} f_i(w) .$  This means, that at the optimal point $w^*$, the loss over every data point $f_i(w)$ is also minimized. 

Because of this, the noise radius is zero for models that satisfy interpolation and we obtain the following result.
\begin{corollary}
Let the assumptions and notations of Theorem \ref{T:cvconvexsmooth} be in place. If interpolation holds, i.e.\ $\sigma=0$, then
\begin{equation*}
    \mathbb{E}[
    f(\bar w^T) - \inf f  
    ]
    \leq
    \frac{\tfrac{1}{\lambda}\norm{w^{0} -w^*}^2 + 
    \tfrac{1}{\delta} \sum\limits_{i=1}^n
    (s_i^0 - \inf f_i)^2
    }{2\gamma(1-\gamma)(1- \lambda L_{\max} )T}
    .
\end{equation*}
\end{corollary}
\section{Numerical experiments}
\label{sec:numerics}
In the case of full batch sampling \FUVAL{} is a new adaptive gradient method given by
\begin{align}
    \tau_t &:= \min\Big\{ c, \frac{\big(f(w^t) - s^t + \delta_t\big)_+}{\lambda_t\|\nabla f(w^t)\|^2+\delta_t} \Big\}, \nonumber\\
    w^{t+1} &= w^t - \tau_t  \lambda_t\nabla f(w^t), \nonumber\\
     s^{t+1} &=  s^t - \delta_t + \tau_t \delta_t, \label{eq:prox-lin-method-ful}
\end{align}
where $s^t$ should converge to $\inf f$. We test four distinct datasets from  LIBSVM~\cite{chang2011libsvm}, namely \texttt{mushrooms},  \texttt{ijcnn1},  \texttt{colon},  and \texttt{covtype}.
We will consider three different choices for the parameters $(\delta,\lambda)$:
\begin{itemize}
    \item a \emph{naive} setting, where $\delta= \lambda = c$ for some $c>0$ ;
    \item a \emph{unit invariant function value} setting, where $\delta = c f(w^0)$ and $\lambda = \frac{c}{f(w^0)}$, for some $c>0$ ;
    \item a \emph{unit invariant gradient} setting, where $\delta = c f(w^0)$ and $\lambda = \frac{c f(w^0)}{\Vert \nabla f(w^0) \Vert^2}$, for some $c>0$.
\end{itemize}
For the full batch experiments, we perform a grid search over $c$, and for each value we run the algorithms for $200$ iterations. For each experiment, we run the algorithms for $20$ epochs (the number of iterations is equal to $20$ times the number of data points in the dataset), and compare the \FUVAL{} methods to GD (gradient descent). For gradient descent we do a grid search over the step size. We then compare the sensitivity of the methods with respect to their only tunable parameter ($c$ or stepsize). We refer to this tunable parameter as the stepsize\_factor, and plot it against suboptimality $f(w^T)-f^*$ after 20 epochs, see Figure~\ref{fig:gd vs fuval}. We found that the variants of \FUVAL{} enjoy a wider settings of good parameters on the \texttt{colon} and \texttt{covtype} data sets, but had a very comparable sensitivity to GD on \texttt{mushrooms} and \texttt{ijcnn1}.

We also compared \FUVAL{} to SGD in Figure~\ref{fig:sgd vs fuval}. But here we found that SGD tended to be less sensitive to tuning its stepsize. We conjecture that this is because we used single element sampling, and thsu the $\alpha_i$'s in \FUVAL{} are re-visited too infrequently and thus become stale.


\def\mywidth{0.45\textwidth} 

\begin{figure}
    \centering
    \includegraphics[width=\mywidth]{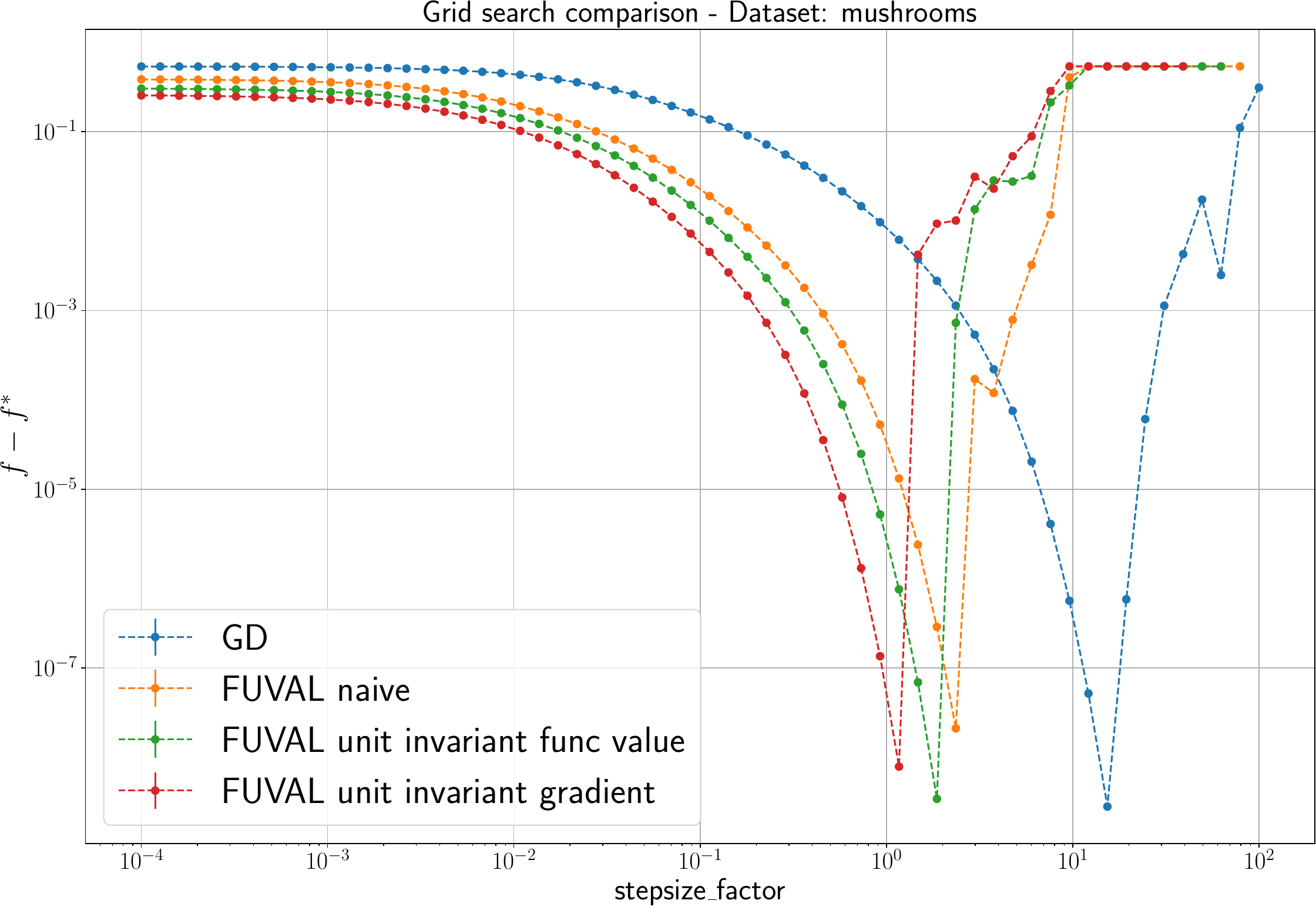}
    \hfill
    \includegraphics[width=\mywidth]{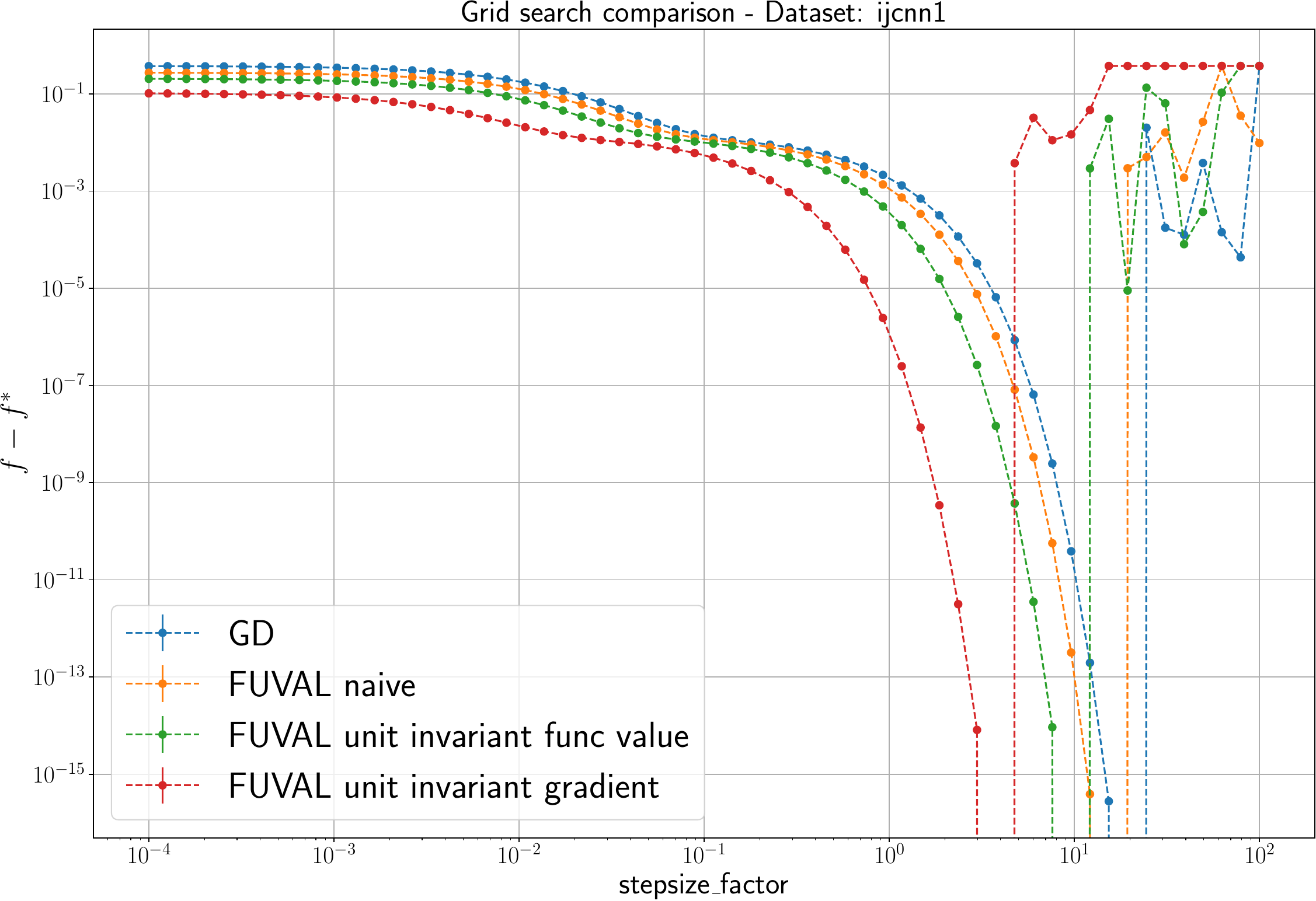} \\
    \includegraphics[width=\mywidth]{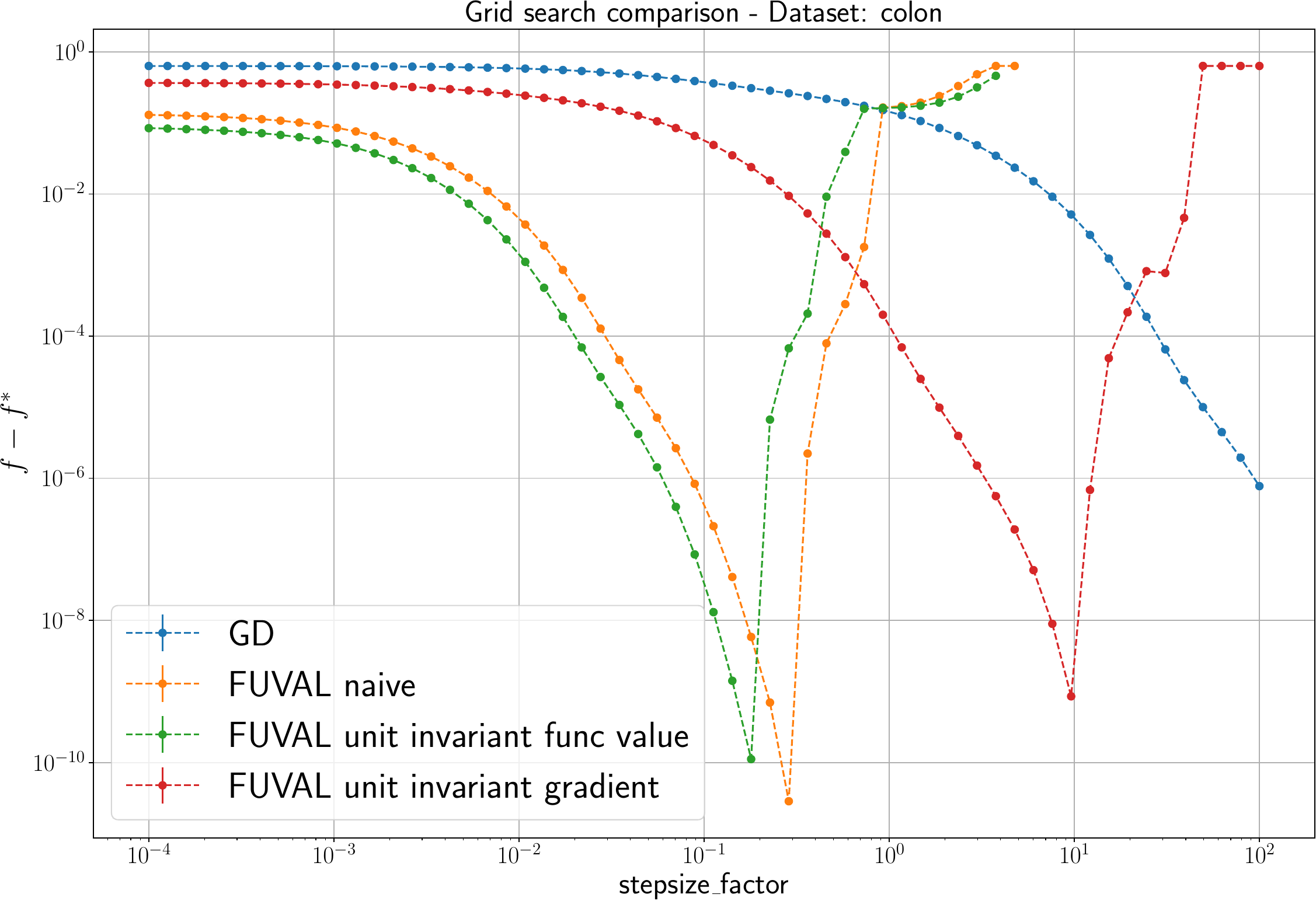}
    \hfill
    \includegraphics[width=\mywidth]{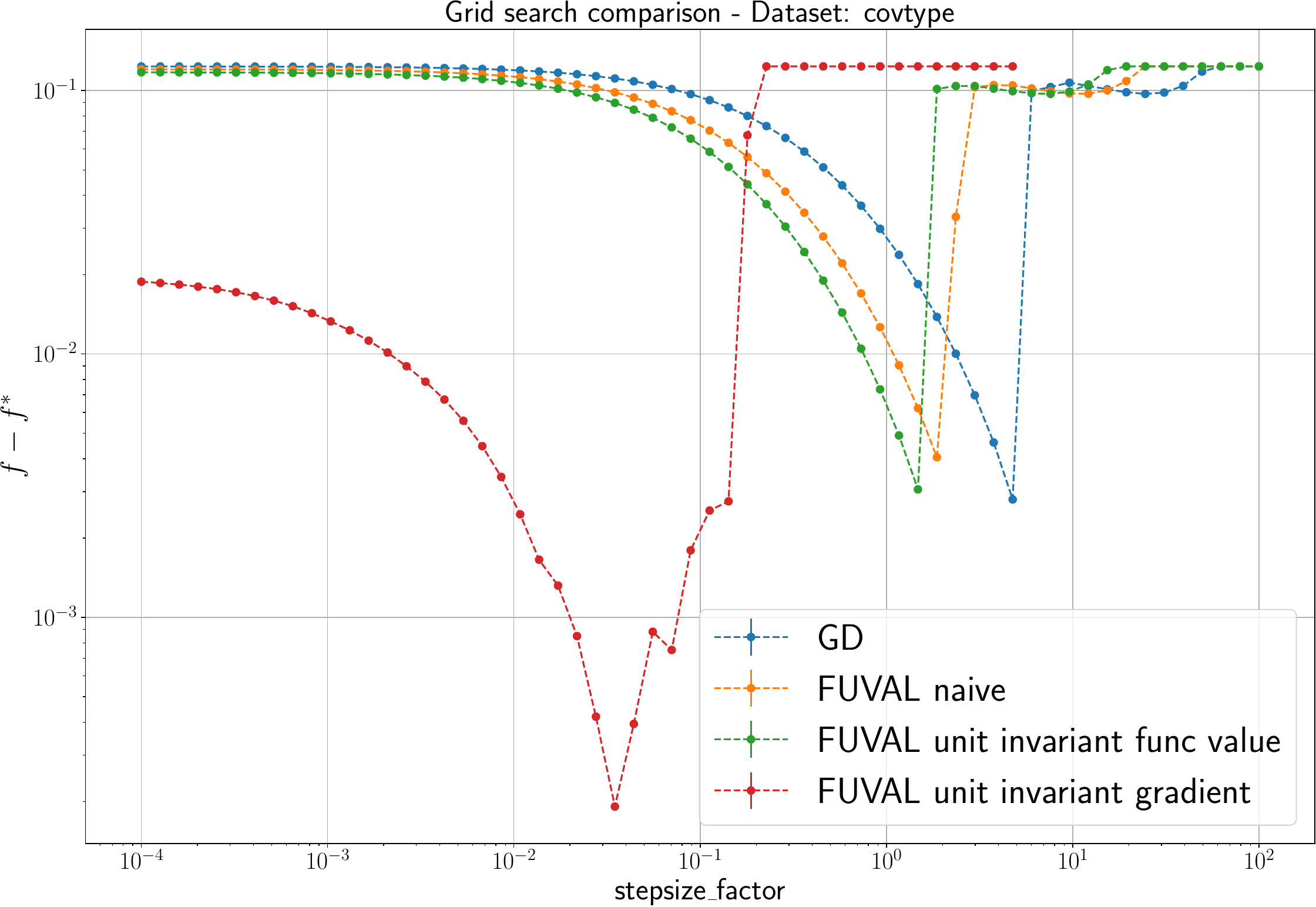}    
    \caption{Comparing GD vs. full batch \FUVAL{} in terms of sensitivity to their only tunable parameter: $c$ for \FUVAL{} and stepsize for GD.}
    \label{fig:gd vs fuval}
\end{figure}

\begin{figure}
    \centering
    \includegraphics[width=\mywidth]{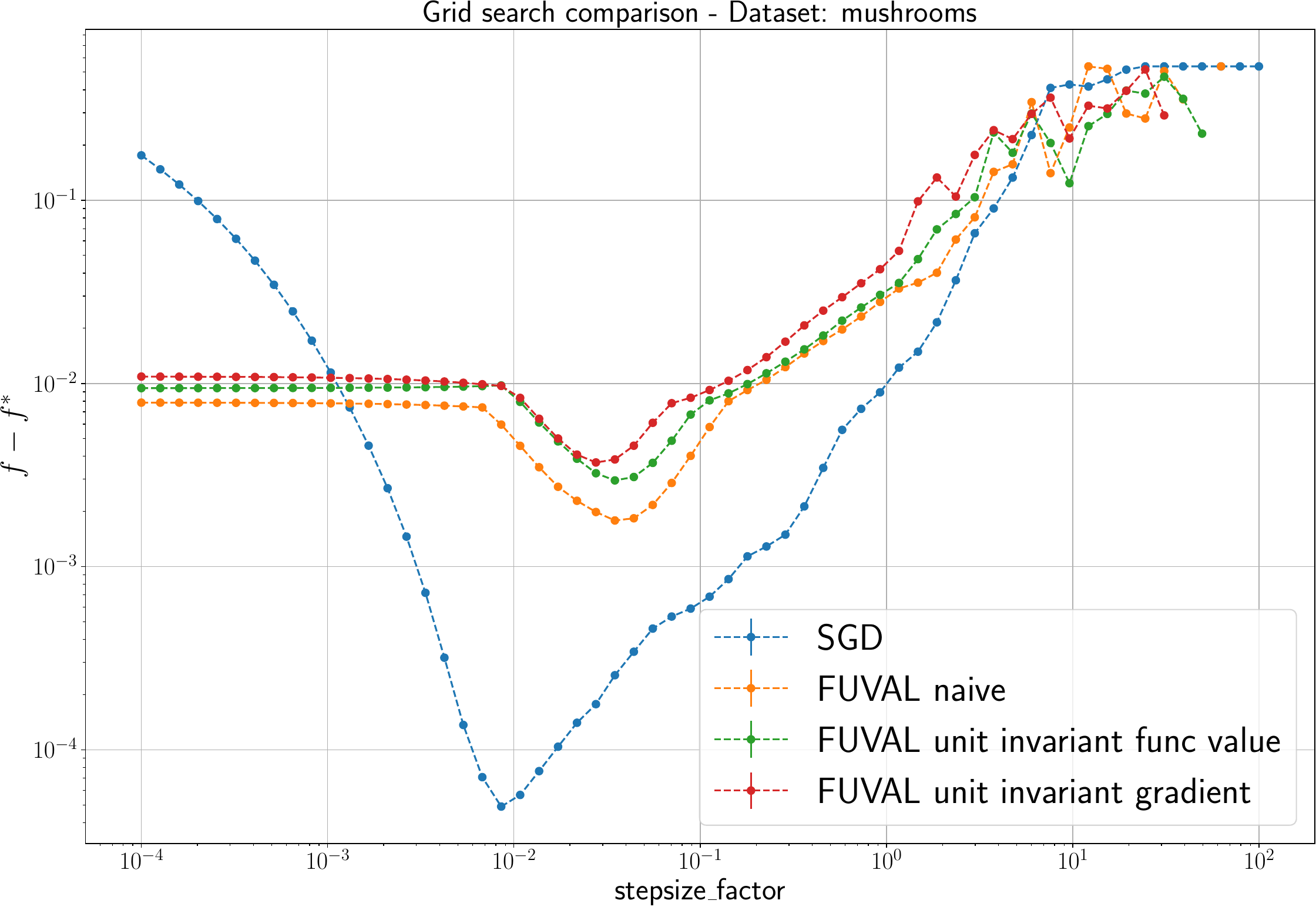}
    \hfill
    \includegraphics[width=\mywidth]{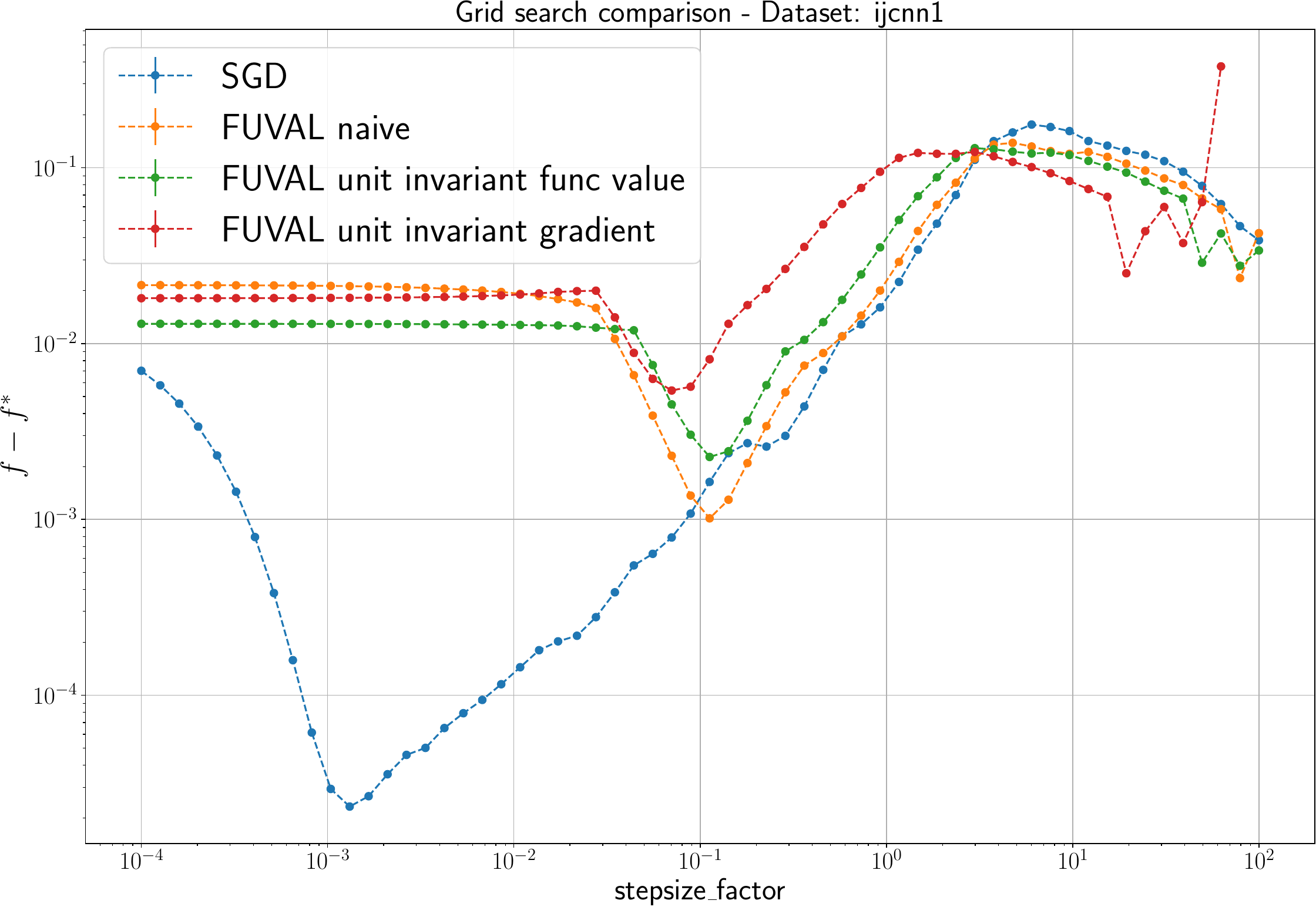} \\
    \includegraphics[width=\mywidth]{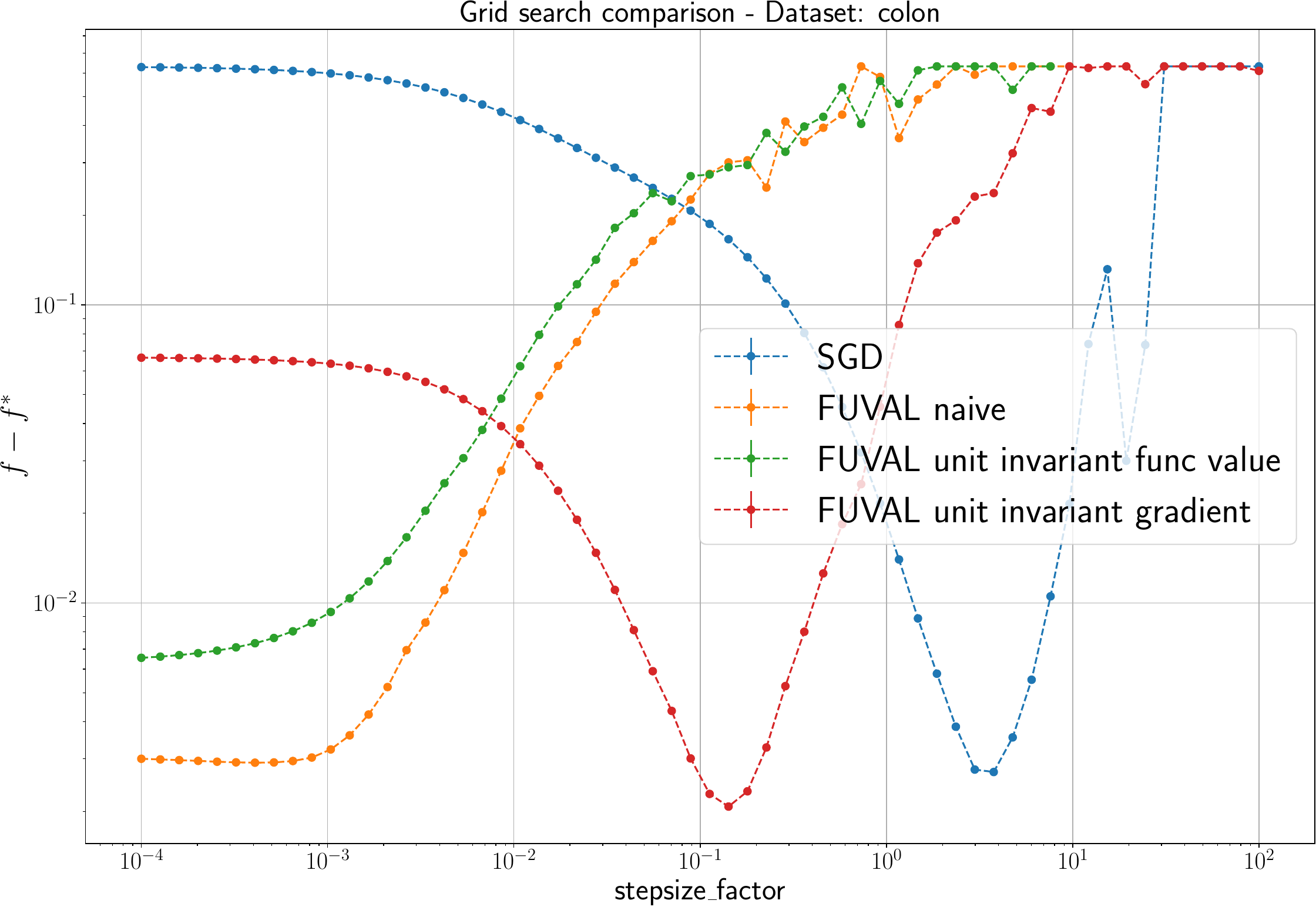}
    \hfill
    \includegraphics[width=\mywidth]{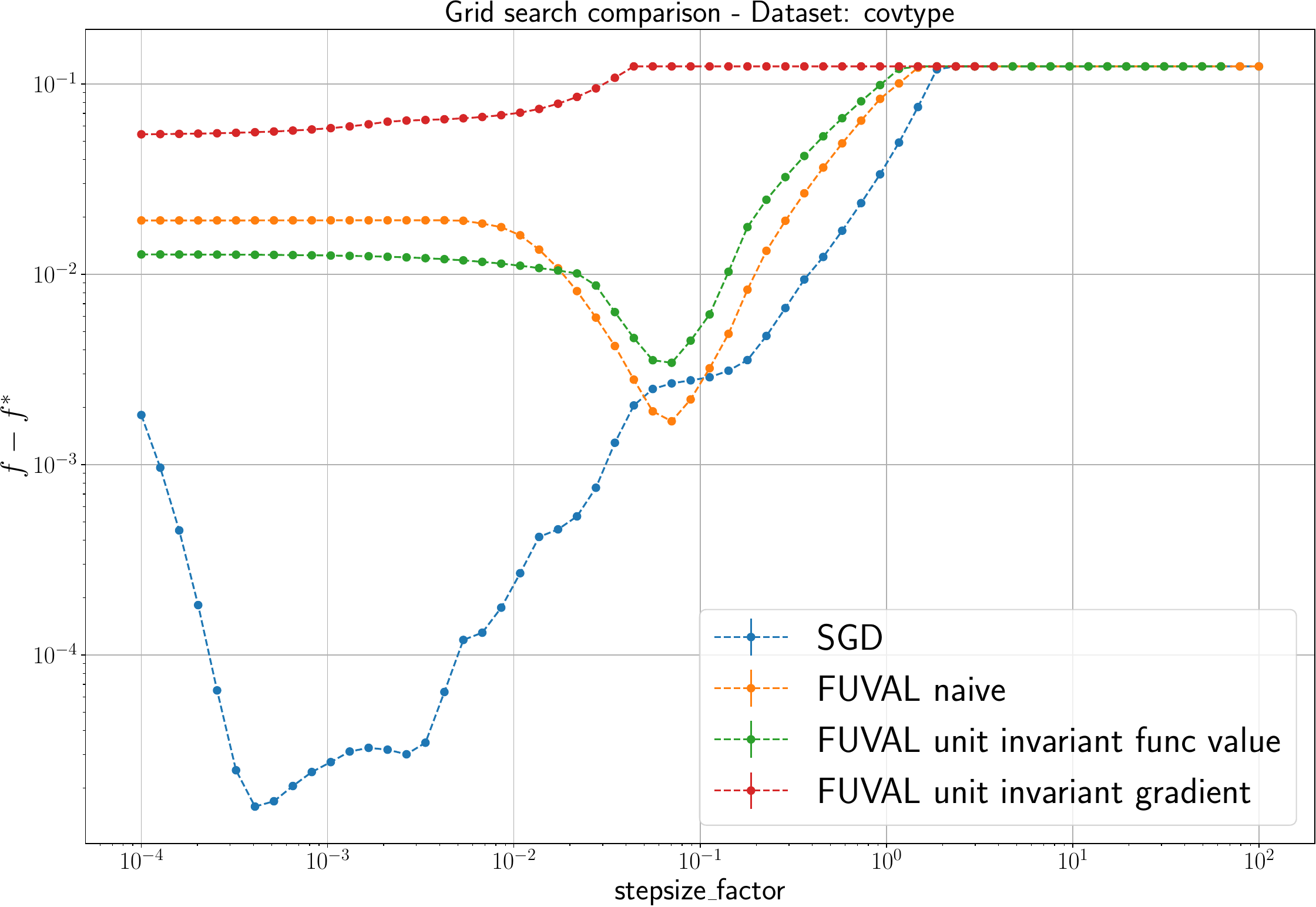}    
    \caption{Comparing SGD vs. \FUVAL{} in terms of sensitivity to their only tunable parameter: $c$ for \FUVAL{} and stepsize for SGD.}
    \label{fig:sgd vs fuval}
\end{figure}


\renewcommand*{\bibfont}{}
{ 
\printbibliography
}

\appendix 

\tableofcontents

\appendix


\section{Auxiliary Lemmas}

This following lemma is taken from  Proposition 3 in~\cite{Khaled-nonconvex-2020}.

\begin{lemma}[Smoothness inequality] \label{lem:convsmoothinter}
Let $f_i:\mathbb{R}^d\to \mathbb{R}$ be $L_i$--smooth and bounded from below, and let $L_{\max} = \displaystyle \max_{i=1,\ldots, n} L_i$. Then, for all $w\in\mathbb{R}^d$ and $i=1,\dots,n$ it holds 
\begin{equation}
    \label{eq:convsmoothinter}
    f_i(w) - \inf f_i \geq \frac{1}{2L_i} \norm{\nabla f_i (w)}^2 \enspace.
\end{equation}
Consequently, if $i \sim \frac{1}{n}$ then taking expectation we have that for any $w\in \mathbb{R}^d$ it holds
\begin{equation}
    \label{eq:convsmoothinterE}
    \E{\norm{\nabla f_i (w)}^2} \leq 2L_{\max}( f(w) -\inf f + \inf f- \E{\inf f_i}) \enspace.
\end{equation}
If all $f_i$ are additionally convex, and if $w^* \in \argmin f(w)$, then
\begin{equation}
    \label{eq:convsmoothE}
    f(w)-f(w^*) = \E{f_i(w) -f_i(w^*)} \geq \frac{1}{2L_{\max}} \E{\norm{\nabla f_i (w)-\nabla f_i (w^*)}^2} \enspace.
\end{equation}
\end{lemma}
\begin{proof}
Using Lemma 5.7 in~\cite{Beck2017} for $f_i$ , we obtain that for all $u,w \in \mathbb{R}^d$ it holds
\begin{align*}
   \inf f_i\leq f_i(u) \leq f_i(w) + \langle \nabla f_i(v),u-w\rangle + \frac{L_i}{2}\|u-w\|^2.
\end{align*}
Minimizing the right-hand side over $u$, the minimum is attained at $\hat u = w- \frac{1}{L_i}\nabla f_i(w)$. Plugging in yields
\begin{align*}
   \inf f_i\leq f_i(w) - \langle \nabla f_i(w), \frac{1}{L_i}\nabla f_i(w) \rangle + \frac{L_i}{2}\|\frac{1}{L_i}\nabla f_i(w)\|^2 = f_i(w) - \frac{1}{2L_i}\|\nabla f_i(w)\|^2.
\end{align*}
Using that $L_i\leq L_{\max}$ for all $i=1,\dots,n$, we have $\|\nabla f_i(w)\|^2 \leq 2L_{\max}(f_i(w)-\inf f_i)$. To prove \eqref{eq:convsmoothinterE}, we simply add and subtract $\inf f$ then take expectation and note that $\E{f_i(w)} = f(w)$.\\
%
Finally in order to proof~\eqref{eq:convsmoothE}, by assumption $f_i$ is convex and $L_{\max}$--smooth.
Applying inequality (2.1.10) of Theorem 2.1.5 in~\cite{nesterov2013introductory} for $f_i$ and substituting $x \leftarrow w^*$ and $y \leftarrow w$, we obtain
\[   f_i(w) - f_i(w^*) \geq \dotprod{\nabla f_i (w^*), w-w^*} + \frac{1}{2L_{\max}} \norm{\nabla f_i (w) - \nabla f_i (w^*)}^2. \]
Taking expectation and using that 
\[\E{\dotprod{\nabla f_i (w^*), w-w^*}}= \dotprod{\nabla f (w^*), w-w^*} =0 , \]
give the result~\eqref{eq:convsmoothE}.
\end{proof}

\begin{lemma}\label{lem:max-of-lin-update}
Let $y,a\in \mathbb{R}^p$ and $c\in \mathbb{R}$. The solution to 
\begin{align*}
    y^+ = \arg \min_y \Big(c+\langle a, y-y^0\rangle\Big)_+  + \frac{1}{2\beta} \|y-y^0\|^2
\end{align*}
is given by 
\[y^+ = y^0 - \min \big\{\beta, \frac{(c)_+}{\|a\|^2}\big\}a.\]
\end{lemma}
\begin{proof}
See Proposition 3 in~\cite{ALI-G}.
\end{proof}


\begin{lemma}[Projection Inequality Constraint] \label{lem:L2ineqconst}
Let $ c\in \R$ and $w,w^0,a\in \R^d$ .
The closed form solution to 
\begin{align}
    \label{eq:L2ineqconstproj}
    w'=& \argmin_{w\in\R^d } \norm{w - w^0}^2  \nonumber \\
    &\, \mbox{subject to } a^\top (w-w^0) +c \leq 0 \enspace,
\end{align}
is given by 
\begin{align} \label{eq:L2ineqconstprojsolw}
w' & =  w^0 -  \frac{(c)_+}{  \norm{a}^2} a , 
\end{align}
where we denote $(x)_+ =\begin{cases}  x & \mbox{ if } x \geq 0 \\ 0 & \mbox{otherwise} \end{cases} .$ 
\end{lemma}
\begin{proof}
If $c \leq 0$ then clearly $w=w^0$ is the solution.  Otherwise, the solution is given by projecting onto the affine space $ a^\top (w-w^0) +c = 0.$ The solution to this projection is given by the pseudoinverse,  that is
\[w' = w^0 -c a^{\dagger} = w^0 -\frac{c}{\norm{a}^2} a. \qed  \]

\end{proof}

\begin{lemma}[Projection Inequality Constraint with slack] \label{lem:slackL2ineqconst}
Let $\delta >0, c\in \R$ and $w,w^0,a\in \R^d$ .
The closed form solution to 
\begin{align}
    \label{eq:slackL2ineqconstproj}
    w',s' =& \argmin_{w\in\R^d, s \in \R } \norm{w - w^0}^2 + \delta (s-s^0)^2 \nonumber \\
    &\, \mbox{subject to } a^\top (w-w^0) +c \leq s \enspace,
\end{align}
is given by 
\begin{align} \label{eq:slackL2ineqconstprojsolw}
w' & =  w^0 - \delta \frac{(c-s^0)_+}{ 1 +\delta \norm{a}^2} a , \\
s' & = s^0+   \frac{(c-s^0)_+}{ 1 +\delta \norm{a}^2}, \label{eq:slackL2ineqconstprojsolb}
\end{align}
where we denote $(x)_+ =\begin{cases}  x & \mbox{ if } x \geq 0 \\ 0 & \mbox{otherwise} \end{cases} .$ 
\end{lemma}
\begin{proof}
The proof is given in Lemma C.2 in~\cite{slackpolyak2022}.  But for completeness we give an outline of the proof here.

The problem~\eqref{eq:slackL2ineqconstproj} is an L2 projection onto a halfspace. 
The solution depends if the projected vector $(w,s) =(w^0,s^0)$ is in the halfspace.

If $w = w^0$ and $s=s^0$ satisfies in the linear inequality constraint, that is if $c \leq s^0$, in which case the solution is simply $w' = w^0$ and $s' = s^0.$

Else, $(w^0,s^0)$ is out of the feasible set, that is $c > s^0$, then we need to project $(w^0,s^0)$ onto the boundary of the halfspace, which means project onto 
\[ \{ (w, s) \in \R^d \times \R^d \, | \, a^\top (w - w^0) + c = s \}  \enspace. \]
In which case the solution is given in by the equality constrained version
\begin{align}
    \label{eq:o8oh48hs48hz4}
    w',s' =& \argmin_{w\in\R^d, s \in \R^b } \norm{w - w^0}^2 + \delta \norm{s-s^0}^2 \nonumber \\
    &\, \mbox{subject to } a^\top (w-w^0) +c = s \enspace,
\end{align}
which is 
\begin{align} \label{eq:tseltjhoets}
w' & =  w^0 - \delta \frac{c-s^0}{ 1 +\delta \norm{a}^2} a , \\
s' & = s^0+   \frac{c-s^0}{ 1 +\delta \norm{a}^2}. \label{eq:tenso8tens}
\end{align}
Combining both cases gives~\eqref{eq:slackL2ineqconstprojsolw}--\eqref{eq:slackL2ineqconstprojsolb}.
\end{proof}

\section{Missing Proofs}

Here we give the missing proofs.  All the associated statements are also repeated for ease of reference.


\subsection{Proof of Lemma~\ref{lem:projupdate} }

\projupdate*
\begin{proof}
The proof follows by re-writing the objective function of the projection~\eqref{eq:funcvallearn} as 
\begin{align*}
       s_j +  \frac{1}{2\lambda}\norm{w- w^t}^2+\frac{1}{2\delta} (s_j- s^t_j)^2 &=
 \frac{1}{2\lambda}\norm{w- w^t}^2 + \frac{1}{2\delta}(s_j -s^t_j+\delta )^2 +\mbox{constants}.
\end{align*}
Thus we can re-write the projection~\eqref{eq:funcvallearn} as
\begin{align}
w^{t+1}, s^{t+1} = & \underset{w \in \R^d, s_j \in \mathbb{R}^n}{\rm{argmin}}~   \norm{w- w^t}^2+\frac{\lambda}{\delta} (s_j- (s_j^t - \delta))^2  + \frac{\lambda}{\delta} \sum_{i \neq j}(s_i -s_i^t)^2\nonumber\\
&\mbox{ subject to }f_j(w^t)+\dotprod{\nabla f_j(w^t),w -w^t } \leq  s_j, \label{eq:funclearnprojgrad}
\end{align}
This projection is now separable in the slack variables $s_i$ for $i \neq j.$ Indeed, for these slack variables the solution is simply $s_i^{t+1} = s_i^t$ for every $i \neq j.$ 

Consequently we can now apply Lemma~\ref{lem:slackL2ineqconst}
where $\delta \leftarrow \frac{\lambda}{\delta}$
and $s^0 \leftarrow s^t-\delta$,
which gives the solution
\begin{align} 
w^{t+1} & =  w^t -\lambda  \frac{(f_j(w^t)-s_j^t+\delta)_+}{ \delta +\lambda\norm{\nabla f_j(w^t)}^2} \nabla f_j(w^t) \label{eq:funcvallearnsolw}\\
s_j^{t+1} & = s_j^t-\delta+   \delta\frac{(f_j(w^t)-s_j^t+\delta)_+}{ \delta +\lambda\norm{\nabla f_j(w^t)}^2}.\label{eq:funcvallearnsols}
\end{align}
\end{proof}

\subsection{Proof of Lemma~\ref{lem:equivpenalty}}

\equivpenalty*
\begin{proof}
first note that if $f_i$ are assumed to be convex, then \eqref{prob:pos-part} is a convex problem, as $(\cdot)_+$ is nondecreasing and convex.
We derive the necessary first-order optimality conditions of \eqref{prob:pos-part} (cf.\ \cite[Thm.\ 3.63]{Beck2017}). 
For $(w,s)$, they are
\begin{align}
    u_i &\in \partial (f_i(w)-s_i)_+, \quad i=1,\dots,n, \label{eqn:opt-cond-i}\\
    0 &\in \frac{1}{n}\sum_{i=1}^n c u_i \partial f_i(w), \label{eqn:opt-cond-ii}\\
    0 &= 1- c u_i, \quad i=1,\dots,n. \label{eqn:opt-cond-iii}
\end{align}
Here, we used the fact that $(f_i(w)-s_i)_+ = \max\{f_i(w)-s_i,0\}$ is the pointwise maximum of convex functions and applied \cite[Cor.\ 4.3.2]{HiriartUrruty2001}. 
We do a simple case distinction: 
\begin{enumerate}
    \item If $s_i > f_i(w)$, then $\partial (f_i(w)-s_i)_+ = \{0\}$. In this case, \eqref{eqn:opt-cond-iii} cannot be fulfilled.
    \item  If $s_i < f_i(w)$, then $\partial (f_i(w)-s_i)_+ = \{1\}$. From \eqref{eqn:opt-cond-iii} we have $u_i=1/c$ which implies $c=1$.
    \item  If $s_i = f_i(w)$, then $\partial (f_i(w)-s_i)_+ = [0,1]$. From \eqref{eqn:opt-cond-iii} we have $u_i=1/c$ which implies $c \geq 1$.
\end{enumerate}
At any solution $(w^*,s^*)$, the necessary first-order optimality conditions are fulfilled and hence it must hold $s_i^* \leq f_i(w^*)$ for all $i\in[n]$. Plugging \eqref{eqn:opt-cond-iii} into \eqref{eqn:opt-cond-ii} gives $0\in\frac{1}{n}\sum_{i=1}^n \partial f_i(w^*)$ and hence $0\in \partial f(w^*)$. Hence, $w^*$ is a (global) minimum of $f$ (due to convexity).
Moreover, as $s_i^* \leq f_i(w^*)$ and $c\geq 1$, we have
\[g(w^*,s^*) = \frac{1}{n}\sum_{i=1}^n s_i^* + c(f_i(w^*) -s_i^*)_+ = \frac{1}{n}\sum_{i=1}^n (1-c)s_i^* + cf_i(w^*).\]
Now, we either have $s_i^* < f_i(w^*)$ in which case $c=1$ and hence 
\[(1-c)s_i^* + cf_i(w^*) = f_i(w^*).\]
Otherwise, $c>1$, but then $s_i^*=f_i(w^*)$ from \eqref{eqn:opt-cond-iii} and hence $(1-c)s_i^* + cf_i(w^*) = f_i(w^*)$. Altogether, we get $g(w^*,s^*)=f(w^*)$.
\end{proof}

\subsection{Proof of Lemma~\ref{lem:modelupdate}}

\modelupdate*
\begin{proof}
Introducing the variables $\hat{s} := \sqrt{\frac{\lambda_t}{\delta_t}}s$ and hence $\hat{s}^t := \sqrt{\frac{\lambda_t}{\delta_t}}s^t$, problem \eqref{eqn:two-scale-update} is equivalent to 
\begin{align}
    w^{t+1}, \hat s^{t+1} = \argmin_{w,\hat s} \sqrt{\frac{\delta_t}{\lambda_t}}\hat s_{j_t} &+ c\Big(f_{j_t}(w^t)+ \langle \nabla f_{j_t}(w^t), w-w^t \rangle - \sqrt{\frac{\delta_t}{\lambda_t}}\hat s_{j_t}\Big)_+  \nonumber\\
    & + \frac{1}{2\lambda_t}\|w-w^t\|^2 + \frac{1}{2\lambda_t}\|\hat s-\hat s^t\|^2.\label{eq:rescaledhats}
\end{align}
Denote $\nu_t := \sqrt{\delta_t/\lambda_t}$. Completing the squares gives 
\[\nu_t\hat s_i+ \frac{1}{2\lambda_t}(\hat s_i- \hat s_i^t)^2 = \frac{1}{2\lambda_t}(\hat s_i - \hat s_i^t + \sqrt{\lambda_t\delta_t})^2 + \mathrm{constants}(\hat s_i),\]
and hence
\begin{align*}
    w^{t+1}, \hat s^{t+1} = &\argmin_{w,\hat s} c\Big(f_{j_t}(w^t)+ \langle \nabla f_{j_t}(w^t), w-w^t \rangle - \nu_t\hat s_{j_t}\Big)_+ + \frac{1}{2\lambda_t}\|w-w^t\|^2 + \frac{1}{2\lambda_t}(\hat s_{j_t}- (\hat s_{j_t}^t - \sqrt{\lambda_t\delta_t}) )^2, \\
    &\text{subject to } \hat s_i = \hat s_i^t, \quad i\neq j_t.
\end{align*}
Applying Lemma \ref{lem:max-of-lin-update} with 
\[y^0 = \begin{bmatrix}w^t\\\hat s_{j_t}^t - \sqrt{\lambda_t\delta_t}\end{bmatrix},~\beta = c\lambda_t,~ a = \begin{bmatrix}\nabla f_{j_t}(w^t) \\ -\nu_t\end{bmatrix},~ c = f_{j_t}(w^t) -\nu_t\hat s_{j_t}^t + \sqrt{\lambda_t\delta_t}\nu_t \]
gives the update 
\begin{align*}
    \hat \tau_t &:= \min\Big\{\lambda_t c, \frac{\big(f_{j_t}(w^t) - \nu_t \hat s_{j_t}^t + \delta_t\big)_+}{\|\nabla f_{j_t}(w^t)\|^2+\nu_t^2} \Big\}, \\
    w^{t+1} &= w^t - \hat \tau_t \nabla f_{j_t}(w^t), \\
    \hat s^{t+1}_j &= \hat s_j^t - \sqrt{\lambda_t\delta_t} + \hat \tau_t \nu_t, \quad \text{if } j = j_t, \\
    \hat s^{t+1}_j &= \hat s_j^t, \quad \text{if } j \neq j_t,
\end{align*}
where we used that $\sqrt{\lambda_t\delta_t}\nu_t = \delta_t$. 
Substituting back $s^t := \nu_t\hat{s}^t$ and using $\nu_t=\sqrt{\delta_t/\lambda_t}$ and $\hat \tau_t = \lambda_t \tau_t$ gives the solution~\eqref{eq:prox-lin-method}.
\end{proof}

\subsection{Proof of Lemma~\ref{cor:fmodelbasedconvhat} }

\fmodelbasedconvhat*
\begin{proof}
Using again the substitution $s^t := \sqrt{\frac{\delta_t}{\lambda_t}}\hat{s}^t  = \sqrt{\frac{\delta}{\lambda}}\hat{s}^t  $ and returning to~\eqref{eq:rescaledhats} have again that 
\begin{align}
    w^{t+1}, \hat s^{t+1} = \argmin_{w,\hat s} \sqrt{\frac{\delta}{\lambda}}\hat s_{j_t} &+ c\Big(f_{j_t}(w^t)+ \langle \nabla f_{j_t}(w^t), w-w^t \rangle - \sqrt{\frac{\delta}{\lambda}}\hat s_{j_t}\Big)_+  \nonumber\\
    & + \frac{1}{2\lambda_t}\|w-w^t\|^2 + \frac{1}{2\lambda_t}\|\hat s-\hat s^t\|^2.\label{eq:rescaledhats2}
\end{align}
Multiplying the objective by $\sqrt{\frac{\lambda}{\delta}}$ and defining $\hat{f}_j = \sqrt{\frac{\lambda}{\delta}}f_j $ and $\rho_t := \sqrt{\frac{\delta}{\lambda}}\lambda_t= \frac{\sqrt{\lambda\delta}}{\sqrt{t+1}}$ gives
\begin{align}
    w^{t+1}, \hat s^{t+1} = \argmin_{w,\hat s} \hat s_{j_t} &+ c\Big(\hat f_{j_t}(w^t)+ \langle \nabla \hat{f}_{j_t}(w^t), w-w^t \rangle - \hat s_{j_t}\Big)_+  \nonumber\\
    & + \frac{1}{2\rho_t}\|w-w^t\|^2 + \frac{1}{2\rho_t}\|\hat s-\hat s^t\|^2.\label{eq:rescaledhats2}
\end{align}
The above is now the model-based method Algorithm \ref{alg:model-based-penalty} with step size $\rho_t$ applied to minimizing
\begin{align}
    \label{prob:pos-part-hat}
    \min_{w \in \mathbb{R}^d,\hat{s}\in \mathbb{R}^n} \frac{1}{n}\sum_{i=1}^n \Big(\hat{s}_i + c (\hat{f}_i(w) - \hat{s}_i)_+ \Big).
\end{align}
Since $f_i$ is $G_i$--Lipschitz, we have that $\hat{f}_i$ is $\sqrt{\frac{\lambda}{\delta}}G_i$--Lipschitz. Consequently by Lemma~\ref{lem:one-sided-model} the model is
 $$\hat{M}_i:= \big(1+c  \sqrt{\frac{\lambda}{\delta}G_i^2 + 1}\big)$$
Lipschitz.  Let $\hat{\mathsf{M}} \eqdef \sqrt{\tfrac{1}{n} \sum_{i=1}^n \hat{M}_i^2 } .$ Since $\rho_t = \frac{\sqrt{\lambda\delta}}{\sqrt{t+1}} =: \frac{\rho}{\sqrt{t+1}},$ we have by Corollary~\ref{cor:fmodelbasedconv} we have the following convergence
\begin{align}\label{eqn:estimate-convex-temp}
    \mathbb{E}\Big[\hat{f}(\bar{w}^{T}) - \hat{f}(w^*)\Big] \leq \frac{\|w^0-w^*\|^2 +\|\hat{s}^0-\hat{s}^*\|^2}{4\rho(\sqrt{T+2}-1) } + \frac{\hat{\mathsf{M}}^2\rho(1+\ln(T+1))}{2\sqrt{T+2}-2}.
\end{align}
Substituting back $s^t = \sqrt{\frac{\delta}{\lambda}}\hat{s}^t$,  $f = \sqrt{\frac{\delta}{\lambda}} \hat{f}$ and $\rho=\sqrt{\delta \lambda} $ gives~\eqref{eqn:estimate-convex-hat}.

\end{proof}

\subsection{Proof of Lemma~\ref{L:common minimizers}}
\commonminimizers*
\begin{proof}
Since both $f$ and $\phi_t$ are convex, their minimizers are exactly their critical points.
The gradient of $\phi_{i,t}$ is given by 
\begin{equation}\label{eq:gradphiit}
 \nabla \phi_{i,t}(w,s) =
 \frac{(f_{i}(w) - s_i + \delta)_+}{\delta + \lambda \Vert \nabla f_{i}(w^t) \Vert^2}
 \begin{pmatrix}
 \nabla f_{i}(w) \\
 -e_i
 \end{pmatrix}
 +
 \begin{pmatrix}
 0 \\ e_i
 \end{pmatrix}
\end{equation}
where $e_j \in \mathbb{R}^n$ denotes the $j$-th vector of the canonical basis.
Let  $(w^*,s^*_t)$ be a critical point of $\phi_t$~\eqref{eq:L1sgd} and thus satisfying
\begin{align}\label{eq:temp1smomsi41}
\frac{1}{n} \sum_{i=1}^n \frac{(f_{i}(w^*) - (s^*_t)_i + \delta)_+}{\delta + \lambda\Vert \nabla f_{i}(w^t) \Vert^2}  \nabla f_{i}(w^*) & = 0,\\
 \frac{(f_{i}(w^*) - (s^*_t)_i + \delta)_+}{\delta + \lambda\Vert \nabla f_{i}(w^t) \Vert^2}  &= 1, \quad \mbox{for }j=1, \ldots, n. \label{eq:temp2smomsi421} 
\end{align}
Inserting the second row of the above into the first gives
\begin{equation}\label{eq:selrtnso84hjr}
\frac{1}{n}  \sum_{i=1}^n  \nabla f_{i}(w^*) =0.
\end{equation}
Consequently $w^*$ is a critical point of $f$.
 
Now let $w^*$ be a critical point of $f$. By choosing $s_{i,t}^*$ given in~\eqref{eq:sstar} it easy to see that~\eqref{eq:temp2smomsi421} holds. Consequently, plugging~\eqref{eq:temp2smomsi421} into~\eqref{eq:temp1smomsi41} and using ~\eqref{eq:selrtnso84hjr}, we conclude that  $(w^*,s^*_t)$ is a critical point of $\phi_{t}$.
\end{proof}

\subsection{Proof of Lemma~\ref{L:online SGD equivalence} }
\onlineSGDequivalence*
\begin{proof}
Recall \eqref{eq:gradphiit}. Therefore, 
\begin{align*}
     \nabla_w \phi_{j_t,t}(w^t,s^t) &= \tau_t \nabla f_{j_t}(w^t) \\ 
     \nabla_{s_{j_t}} \phi_{j_t,t}(w^t,s^t) &= 1- \tau_t, \mbox{ and, } \\
     \nabla_{s_i}  \phi_{j_t,t}(w^t,s^t) &= 0, \quad \mbox{ for }i \neq j_t.
\end{align*}
Plugging the above gradients into~\eqref{eq:sgdview} gives \eqref{alg:relaxedversion}.
\end{proof}

%

\subsection{Proof of Theorem~\ref{T:cvconvexsmooth}}
\cvconvexsmooth*

\begin{proof}
Denote $z^t=(w^t,s^t)$. Consider  the estimate~\eqref{eq:conv1ststep} in Lemma \ref{L:lyapunov estimate}, together with the notation $\bar \sigma_t^*$ in~\eqref{eq:sigmat}.
Reordering the terms in~\eqref{eq:conv1ststep}, summing both sides for $t = 0, \ldots, T-1$  and dividing by $T$ we have (after cancellations due to a telescopic sum of positive terms) that
\begin{equation}\label{eq:tempo8jr5xx}
    \frac{2\gamma(1-\gamma)}{T}
    \sum\limits_{t=0}^{T-1}
    \mathbb{E}[\phi_{t}(z^t) -  \phi_{t}(z^*)]
    \leq
    \frac{1}{T}
    \mathbb{E}[\norm{z^{0} -z^*}^2_{\mD}]
    + \frac{2\gamma^2}{T} 
    \sum\limits_{t=0}^{T-1}
   \bar  \sigma_t^*.    
\end{equation}

Consider now that $w^* \in {\rm{argmin}}~f$ and $s_i^*=f_i(w^*)$.
In particular, from Lemma \ref{L:lyapunov estimate} we have that $\bar \sigma_t^* \leq \sigma$.
Moreover, from the definition of $\phi_t$ in~\eqref{eq:L1sgd} is easy to see that $\phi_t(z^*) \leq \inf f + \frac{\delta}{2}$. Furthermore using Lemma~\ref{lem:lowerphi} we have that

\begin{equation*}
    \phi_t(z^t) \geq 
    \inf f +
    \frac{\delta}{2} - \lambda L_{\max}\sigma
    +(1- \lambda L_{\max} )(f(w^t) - \inf f).
\end{equation*}
Using these observations in~\eqref{eq:tempo8jr5xx} we have 
\begin{equation*}
    \frac{2\gamma(1-\gamma)}{T}
    \sum\limits_{t=0}^{T-1}
    \mathbb{E}[
    (1- \lambda L_{\max} )(f(w^t) - \inf f) - \lambda L_{\max} \sigma 
    ]
    \leq
    \frac{1}{T}
    \mathbb{E}[\norm{z^{0} -z^*}_{\mD}^2]
    + {2\gamma^2 } \sigma.
\end{equation*}
Reordering the terms, and using Jensen inequality, we get
\begin{equation*}
    {2\gamma(1-\gamma)(1- \lambda L_{\max} )}
    \mathbb{E}[
    f(\bar w^T) - \inf f  
    ]
    \leq
    \frac{1}{T}
    \mathbb{E}[\norm{z^{0} -z^*}_{\mD}^2]
    +2\gamma  \sigma (\gamma + \lambda L_{\max}(1-\gamma))
    .
\end{equation*}
Use  the fact that $\gamma <1$ and $2 \lambda L_{\max} <1$ so that we can divide by nonzero constants, and finally obtain
\begin{equation*}
    \mathbb{E}[
    f(\bar w^T) - \inf f  
    ]
    \leq
    \frac{1}{2\gamma(1-\gamma)(1- \lambda L_{\max} )T}
    \mathbb{E}[\norm{z^{0} -z^*}_{\mD}^2]
    +  \sigma \frac{\gamma + \lambda L_{\max}(1-\gamma)}{(1-\gamma)(1- \lambda L_{\max})}.
\end{equation*}
\end{proof}

\section{Prox-linear as a Model Based Method}
\label{sec:model-prox}
Here we clarify the connection between the method~\eqref{eq:prox-lin-method} and the prox-linear method.  To do so,  we adopt much of the notation of model based methods~\cite{Davis2019}.  We then show how to adapt the results given in~\cite{Davis2019} to our setting.

Define $h_i:\mathbb{R}^2 \to \mathbb{R}$ and $c_i:\mathbb{R}^{d+n} \to \mathbb{R}^2$ as follows: for $z=(z_1,z_2)\in \R^2$, and $u=(w,s)\in \R^{d+n}$ let
\begin{align*}
    h_i(z) &:= z_2 + c(z_1)_+, \\
    c_i(w,s) &:= \begin{bmatrix} f_i(w) -s_i \\ s_i\end{bmatrix}.
\end{align*}
In a slight abuse of notation, we will write both $c_i(u)$ and $c_i(w,s)$ interchangeably when $u=(w,s)$. With the above, \eqref{prob:pos-part} is equivalent to the problem
\begin{align}\label{prob:penalty}
    \min_{(w,s)\in \mathbb{R}^{d+n}} \frac{1}{n}\sum_{i=1}^n h_i(c_i(w,s)).
\end{align}
Let us define the objective function of \eqref{prob:penalty} as
\[g(w,s) := \frac{1}{n}\sum_{i=1}^n h_i(c_i(w,s)).\]
In the philosophy of mode-based stochastic proximal point \cite{Davis2019}, we construct the prox-linear model of the objective function: for $u=(w,s) \in \R^{d+n}$ and $y=(v,q) \in \R^{d+n}$, define
\begin{align}
    \label{eqn:model_of_g}
    \begin{split}
    g_u(y;i) :&= h_i\big(c_i(u) + \langle \nabla c_i(u), y-u \rangle\big) \\
    &= q_i + c \big(f_i(w) + \langle \nabla f_i(w), v-w \rangle - q_i \big)_+.
    \end{split}
\end{align}
In iteration $t$, if $j_t\in[n]$ is drawn at random, the model-based update is given by
\begin{align*}
    u^{t+1} = \arg \min_u h_{j_t}(c_{j_t}(u^t) + \langle\nabla c_{j_t}(u^t), u-u^t\rangle) + \frac{1}{2\lambda_t}\|u-u^t\|^2,
\end{align*}
where $\lambda_t>0$ is the step size. Rewriting in terms of $u^t = (w^t,s^t)$ and plugging in the definition of $h_i, c_i$ we have~\eqref{prob-model-based-update}.

\begin{align}\label{prob-model-based-update}
    w^{t+1}, s^{t+1} = \arg \min_{w,s} s_{j_t} + c \Big(f_{j_t}(w^t) + \langle\nabla f_{j_t}(w^t), w-w^t\rangle - s_{j_t} \Big)_+ + \frac{1}{2\lambda_t}\Big(\|w-w^t\|^2 + \|s-s^t\|^2\Big).
\end{align}
We now give a closed-form solution for the iterate update \eqref{prob-model-based-update}.
\begin{lemma} \label{lem:unscaled-model-based-update}
The solution to \eqref{prob-model-based-update} is given by 
\begin{align*}
    \tau_t &:= \min\Big\{\lambda_t c, \frac{\big(f_{j_t}(w^t) - s_{j_t}^t + \lambda_t\big)_+}{\|\nabla f_{j_t}(w^t)\|^2+1} \Big\}, \\
    w^{t+1} &= w^t - \tau_t \nabla f_{j_t}(w^t), \\
    s^{t+1}_j &= s_j^t - \lambda_t + \tau_t, \quad \text{if } j = j_t, \\
    s^{t+1}_j &= s_j^t, \quad \text{if } j \neq j_t.
\end{align*}
\end{lemma}
\begin{proof}
For any $i$, we have $s_i + \frac{1}{2\lambda_t} (s_i-s_i^t)^2 = \frac{1}{2\lambda_t}(s_i - s_i^t + \lambda_t)^2 +s_i^t - \frac{\lambda_t}{2}$. Note that $s_i^t - \frac{\lambda_t}{2}$ is constant in $w,s$. Thus, problem \eqref{prob-model-based-update} is equivalent to solving
\begin{align*}
    w^{t+1}, s^{t+1} = &\arg \min_{w,s} c \Big(f_{j_t}(w^t) + \langle\nabla f_{j_t}(w^t), w-w^t\rangle - s_{j_t} \Big)_+ + \frac{1}{2\lambda_t}\Big(\|w-w^t\|^2 + (s_{j_t} - (s^t_{j_t}-\lambda_t))^2\Big)\\
    &\text{subject to}\quad s_j = s_j^t, \quad  j \neq j_t.
\end{align*}
Dividing by $c$ and applying Lemma~\ref{lem:max-of-lin-update} with $y^0 = (w^t, s_{j_t}^t-\lambda_t)$, $a=(\nabla f_{j_t}(w^t), -1)$, $\beta = \lambda_tc$ and $c=f_{j_t}(w^t) - s_{j_t}^t + \lambda_t$, we get the claimed update.
\end{proof}

We next show properties of the model function $g_u(y;i)$.
\begin{lemma}\label{lem:one-sided-model}
Let $P$ be the uniform probability measure on $\{1,\dots,n\}$, i.e.\ $P(\{i\})=\frac{1}{n}$ for all $i\in[n]$. Let further $u=(w,s) \in \mathbb{R}^{d+n}$ and $y=(v,q)\in \mathbb{R}^{d+n}$. For $i\in[n]$ recall 
\begin{align*}
    g_u(y;i) = h_i(c_i(u) + \langle \nabla c_i(u), y-u \rangle) = q_i + c \big(f_i(w) + \langle \nabla f_i(w), v-w \rangle - q_i \big)_+.
\end{align*}
Then, it holds:
\begin{enumerate}
    \item[(B1)] It is possible to generate i.i.d. realizations $j_1,j_2,\dots \sim P$.
    \item[(B2)] We have $\mathbb{E}_{j\sim P}[g_u(u;j)] = \frac1n \sum_{i=1}^n g_u(u;i) = g(w,s)$ for all $u\in \R^{d+n}$. Further, if $f_i$ is convex for all $i\in[n]$, then
    \begin{align*}
        \mathbb{E}_{j\sim P}[g_u(y;j)]=\frac1n \sum_{i=1}^n g_u(y;i) \leq g(v,q) \quad \forall u,y \in \mathbb{R}^{d+n}.
    \end{align*}
    \item[(B3)] The mapping $y\mapsto g_u(y;i)$ is convex for all $u$ and all $i\in [n]$.
    \item[(B4)] If $f_i$ is $G_i$-Lipschitz for all $i\in[n]$, define $M_i:= \big(1+c  \sqrt{G_i^2 + 1}\big)$ and $\mathsf{M}:= \sqrt{\frac1n \sum_{i=1}^n M_i^2}$. Then, 
    \begin{align*}
        g_u(u;i) - g_u(y;i) \leq M_i \|u-y\| \quad \forall u,y \in \mathbb{R}^{d+n}.
    \end{align*}
\end{enumerate}
\end{lemma}
\begin{proof}
\begin{enumerate}
    \item[(B1)]  Evident.
    \item[(B2)] The first statement follows immediately from the definition of $g_u(\cdot;i)$ and $g$. For the second statement, due to convexity of $f_i$ we have
    \begin{align*}
        c_i(u) + \langle c_i(u), y-u \rangle = 
        \begin{bmatrix}
        f_i(w) + \langle \nabla f_i(w), v-w \rangle - q_i \\ q_i \end{bmatrix} 
        \overset{\text{componentwise}}{\leq}
        \begin{bmatrix} f_i(v) - q_i \\ q_i \end{bmatrix} = c_i(y).
    \end{align*}
    Since $h_i$ is monotone in each component, we get 
    \begin{align*}
    g_u(y;i) = h_i(c_i(u) + \langle c_i(u), y-u \rangle) \leq h_i(c_i(y)).
    \end{align*}
    Summing over $=1,\dots,n$ and dividing by $n$ gives the result.
    \item[(B3)] The function $h_i$ is convex as $(\cdot)_+$ is convex. The function $y \mapsto g_u(y;i)$ is a composition of a convex and a linear mapping and therefore convex.
    \item[(B4)] Denote with $e_i$ the $i$-t element of the standard Euclidean basis. Using $(a)_+ - (b)_+ \leq (a-b)_+ $
    we have
    \begin{align*}
        g_u(u;i) - g_u(y;i) &= s_i + c \big(\underbrace{f_i(w) - s_i}_{=:a} \big)_+ - q_i - c \big(\underbrace{f_i(w) + \langle \nabla f_i(w), v-w \rangle - q_i }_{=:b}\big)_+ \\
        &\leq
        s_i - q_i +c\big( \langle \nabla f_i(w), w-v \rangle + q_i-s_i  \big)_+\\
        & = s_i - q_i + c\Big([\nabla f_i(w), -e_i]^T\begin{bmatrix}w - v \\ s-q\end{bmatrix}\Big)_+ \\
        &\leq s_i - q_i + c\|[\nabla f_i(w), -e_i]\|\cdot \|y-u\| \leq \big(1+c\|[\nabla f_i(w), -e_i]\|\big) \|y-u\|.
    \end{align*}
    With $\|[\nabla f_i(w), -e_i]\| = \sqrt{\|\nabla f_i(w)\|^2 + 1} \leq \sqrt{G_i^2 + 1} $, we conclude
    \begin{align*}
        g_u(u;i) - g_u(y;i) \leq \big(1+c  \sqrt{G_i^2 + 1}\big)\|y-u\| \quad \forall x,y.
    \end{align*}
\end{enumerate}
\end{proof}
\begin{corollary}
Let $f_i$ be convex and $G_i$-Lipschitz for all $i\in[n]$. Then, $g_u(y;i)$ is a stochastic one-sided model in the sense of \cite[Assum.\ B]{Davis2019}.
\end{corollary}
\begin{proof}
The statements (B1)-(B4) in Lemma \ref{lem:one-sided-model} coincide with (B1)-(B4) in \cite[Assum.\ B]{Davis2019} for (in the notation of \cite{Davis2019}) $r=0$, and $\tau=\eta=0$, $\mathsf{L} = \mathsf{M}$.
\end{proof}
\subsection{Convergence analysis}
%
%
\begin{algorithm}
\begin{algorithmic}[1]
\State {\bf Inputs:}  step sizes $\lambda_t >0,$ penalty multiplier $c \geq 1$.
\State {\bf Initialize:} $w^0 \in\mathbb{R}^d$ and $s_i^0 \in \mathbb{R}$ for $i=1,\ldots, n.$
\For{$t =0,\ldots, T$} 
\State Sample $j_t$ randomly from $[n]$.
\State Compute $\tau_t = \min\Big\{\lambda_t c, \frac{\big(f_{j_t}(w^t) - s_{j_t}^t + \lambda_t\big)_+}{\|\nabla f_{j_t}(w^t)\|^2+1} \Big\}$ and update
\State $\displaystyle    w^{t+1} \;= w^t - \tau_t \nabla f_{j_t}(w^t) $
\State $\displaystyle   
    s^{t+1}_i \;= \begin{cases}
      s_i^t - \lambda_t + \tau_t, \quad &\text{if } i = j_t, \\
      s_i^t, \quad &\text{else.}
    \end{cases}
$
\EndFor
\State {\bf Output:} $w^{T+1}, s^{T+1}$
\end{algorithmic}
\caption{}
\label{alg:model-based-penalty}
\end{algorithm}

\begin{proposition}\label{prop:apply-davis}
Let $f_i$ be convex and $G_i$-Lipschitz for all $i\in[n]$. Let $(w^*,s^*) \in \argmin_{w,s} g(w,s)$. 
Let the iterates $(w^t,s^t)$ be generated by Algorithm \ref{alg:model-based-penalty} with step sizes $\lambda_t>0$. Then, it holds
\begin{align}\label{eqn:davis-ineq-convex}
    2\lambda_t \mathbb{E}\Big[g(w^{t+1},s^{t+1}) - g(w^*,s^*)\Big]  \leq \mathbb{E}\|(w^t,s^t) - (w^*,s^*)\|^2  - \mathbb{E}\|(w^{t+1},s^{t+1}) - (w^*,s^*)\|^2 + 2\mathsf{M}^2 \lambda_t^2.
\end{align}
Define $\bar{w}^T := \frac{1}{T+1}\sum_{t=0}^{T}\lambda_t w^{t+1}$ and $\bar{s}^T := \frac{1}{T+1}\sum_{t=0}^{T} \lambda_t s^{t+1}$. Choosing $\lambda_t = \frac{\lambda}{\sqrt{t+1}}$, for some $\lambda>0$, we have
\begin{align}\label{eqn:davis-convex-anytime}
    \mathbb{E}\Big[g(\bar{w}^{T},\bar{s}^{T}) - g(w^*,s^*)\Big] \leq \frac{\|w^0-w^*\|^2 +\|s^0-s^*\|^2}{4\lambda(\sqrt{T+2}-1) } + \frac{\mathsf{M}^2\lambda(1+\ln(T+1))}{2\sqrt{T+2}-2}.
\end{align}
Choosing $\lambda_t = \frac{\lambda}{\sqrt{T+1}}$ instead, we get
\begin{align}\label{eqn:davis-convex-constant}
    \mathbb{E}\Big[g(\bar{w}^{T},\bar{s}^{T}) - g(w^*,s^*)\Big] \leq \frac{\|w^0-w^*\|^2 +\|s^0-s^*\|^2}{2\lambda\sqrt{T+1}} + \frac{\mathsf{M}^2\lambda}{\sqrt{T+1}} = \mathcal{O}(\frac{1}{\sqrt{T+1}}),
\end{align}
\end{proposition}
\begin{proof}
We apply the theory of \cite{Davis2019} with $\varphi=g$, $x^t=(w^t,s^t)$, $\beta_t=\lambda_t^{-1}$ and $\bar \rho=\eta=\tau=0$.
Using (4.8) of \cite[Lem.\ 4.2]{Davis2019} with $x=(w^*,s^*)$ and taking expectation yields \eqref{eqn:davis-ineq-convex}. 
Now summing \eqref{eqn:davis-ineq-convex} from $t=0,\dots,T$ and dividing by $\sum_{t=0}^{T} \lambda_t$ yields
\begin{align}
    \tfrac{1}{\sum_{t=0}^{T}\lambda_t} \sum_{t=0}^{T} \lambda_t \mathbb{E}\Big[g(w^{t+1},s^{t+1}) - g(w^*,s^*)\Big] \leq \frac{\|w^0-w^*\|^2 +\|s^0-s^*\|^2}{2\sum_{t=0}^T\lambda_t} + \frac{\mathsf{M}^2 \sum_{t=0}^T\lambda_t^2}{\sum_{t=0}^T\lambda_t}.\label{eq:tempxo8lx5x5}
\end{align}
Using convexity of $g$ and Jensen's inequality we can estimate the left-hand side from below by $\mathbb{E}\Big[g(\bar{w}^{T},\bar{s}^{T}) - g(w^*,s^*)\Big]$. Using the integral bound, 
\[ \int_{N}^{M+1} h(x) dx \; \leq \;  \sum_{n=N}^M h(n)\; \leq \;  h(N)+ \int_{N}^{M} h(x) dx, \quad \mbox{for every decreasing }h(x), \mbox{ and every } N, M \in \N,  \]
we have that
\begin{align*}
    &\sum_{t=0}^{T} \tfrac{1}{\sqrt{t+1}} \geq \int_0^{T+1} \tfrac{1}{\sqrt{s+1}} ds = 2\sqrt{T+2} -2 ,\\
    &\sum_{t=0}^{T} \tfrac{1}{t+1} \leq 1+\int_0^{T} \tfrac{1}{s+1} ds = 1+\ln (T+1).
\end{align*}
 Plugging in these estimates gives \eqref{eqn:davis-convex-anytime}. 
 This same expression~\eqref{eqn:davis-convex-anytime} is essentially also given in (4.17) of \cite[Thm.\ 4.4]{Davis2019}.
 The last claim~\eqref{eqn:davis-convex-constant} follows by plugging in $\lambda_t = \lambda/\sqrt{T+1}$ into~\eqref{eq:tempxo8lx5x5}. 
\end{proof}
Now try to connect the values of $f$ and $g$.
\begin{lemma}\label{lem:f-and-g}
Let $c\geq 1$ for all $i\in[n]$ and let $(w,s)\in \mathbb{R}^{d+n}$. Then, it holds $g(w,s) \geq f(w)$.
\end{lemma}

\begin{proof} 
We first show that $s_i + c(f_i(w) -s_i)_+ \geq f_i(w)$ for all $i\in[n]$ by case distinction:
\begin{enumerate}
    \item Assume $f_i(w) -s_i \leq 0$. Then $(f_i(w) -s_i)_+ = 0$ and 
    \[s_i + c(f_i(w) -s_i)_+ = s_i \geq f_i(w),\]
    where the last inequality follows from the assumption.
    \item Assume $f_i(w) - s_i > 0$. 
    Then $s_i + c(f_i(w) -s_i)_+ = (1-c) s_i + cf_i(w)$. Now $c\geq 1$ implies $1-c \leq 0$ and hence $(1-c) s_i \geq (1-c) f_i(w)$. Altogether,
    \[s_i + c(f_i(w) -s_i)_+ = (1-c) s_i + c f_i(w) \geq (1-c) f_i(w) + c f_i(w) = f_i(w). \]
\end{enumerate}
Applying this for all $i\in[n]$, we get
\begin{align*}
    g(w,s) &= \frac{1}{n}\sum_{i=1}^n s_i + c(f_i(w) -s_i)_+ \geq \frac{1}{n}\sum_{i=1}^n f_i(w) =  f(w).
\end{align*}
\end{proof}
\begin{corollary} \label{cor:fmodelbasedconv}
Let the assumptions of Proposition \ref{prop:apply-davis} hold. If $c\geq1$ for all $i\in[n]$, and $\lambda_t=\frac{\lambda}{\sqrt{t+1}}$, then we have
\begin{align}\label{eqn:estimate-convex}
    \mathbb{E}\Big[f(\bar{w}^{T}) - f(w^*)\Big] \leq \frac{\|w^0-w^*\|^2 +\|s^0-s^*\|^2}{4\lambda(\sqrt{T+2}-1) } + \frac{\mathsf{M}^2\lambda(1+\ln(T+1))}{2\sqrt{T+2}-2}.
\end{align}
\end{corollary}
\begin{proof}
As $c\geq1$ for all $i\in[n]$, we can apply Lemma \ref{lem:f-and-g} to get $f(\bar{w}^T) \leq g(\bar{w}^T, \bar{s}^T)$ almost surely and Lemma \ref{lem:equivpenalty} to get $f(w^*) = g(w^*, s^*)$. Together with \eqref{eqn:davis-convex-anytime}, we conclude the proof.
\end{proof}
%
%
%

\section{Convergence Theorem through SGD Viewpoint}

In this section we will make use of the following assumptions.
\begin{assumption}[Convexity]\label{ass:convexity}
Let $f_i$ be convex for all $i\in \{1,\dots , n\}$.
\end{assumption}
\begin{assumption}[Smoothness]\label{ass:smoothness}
Let $f_i$ be $L_i$--Lipschitz smooth for all $i\in \{1,\dots , n\}$, meaning that $\nabla f_i$ is $L_i$--Lipschitz continuous. Let $L_{\max}:= \max_{i=1,\dots,n} L_i$. In this case, $f$ is $L$--Lipschitz smooth for some $L\leq L_{\max}$.
\end{assumption}
\begin{assumption}[Lipschitz continuity]\label{ass:lipschitz}
Let $f_i$ be $G_i$--Lipschitz continuous for all $i\in \{1,\dots , n\}$, and define $G^2 := \sum_i G_i^2$.
\end{assumption}

First we prove that each slack variable is lower bounded by the infinum of the function it is tracking. 
\begin{restatable}[Lower bound on the slack variables]{lemma}{lowerboundslackvariable}
\label{L:lower bound slack variable}
 Consider the iterates $(w^t,s^t)$ given by~\eqref{alg:relaxedversion} and let
\begin{equation}
     \hat{s}_{j_t}^{t+1}   \; = \; s_{j_t}^{t} +\delta (\tau_t -1).  
\end{equation}
 Let  Assumption \ref{ass:smoothness} hold and let $\lambda \in (0, \tfrac{1}{2L_{\max}}]$. If   $s_i^0 \geq \inf f_i$  for every $i \in \{1, \dots, n\}$, 
then  $s_i^t \geq \inf f_i$ and $\hat s_i^t \geq \inf f_i$ for every $t \in \mathbb{N}$ and $i \in \{1, \dots, n\}$
\end{restatable}
\begin{proof}

Let $j = j_t$ be the index sampled at the iteration $t+1$, and let $\tau_t$ be the corresponding stepsize of \eqref{alg:relaxedversion}.
We have, for every $i \neq j$, that $s_i^{t+1} = s_i^t$, thus any hypothesis on $s_i^t$ carries over immediately for $s_i^{t+1}.$

If $\tau_t =0$, then $s_j^t\geq f_j(w^t)+\delta$. Hence, 
$$s_j^{t+1} = s_j^t -\gamma \delta \geq f_j(w^t)+\delta(1-\gamma)\geq f_j(w^t)  \geq \inf f_j, $$
regardless of the induction hypothesis. 

As for $\hat{s}_j^{t+1}$ we have that 
$$\hat s_j^{t+1} = s_j^t - \delta \geq f_j(w^t)\geq \inf f_j. $$

Now assume that $s_i^0 \geq \inf f_i $  and that $f_i$ is $L_i$--smooth for every $i \in \{ 1,\ldots, n\}.$
We have 
\[ s_j^{t+1} = s_j^t + \gamma \delta (\tau_t-1)=(1-\gamma)s_j^t+\gamma (s_j^t+\delta\tau_t-\delta).\]
Moreover, it holds
\begin{align*}
    s_j^t+\delta\tau_t-\delta &=  s_j^t +  \delta \left( \frac{f_{j}(w^t)  -s_{j}^t+ \delta}{\delta+\lambda \Vert \nabla f_{j}(w^t) \Vert^2} -1\right) \nonumber \\
    &= s_j^t +  \delta \frac{f_{j}(w^t)  -s_{j}^t -\lambda \Vert \nabla f_{j}(w^t) \Vert^2 }{\delta+\lambda \Vert \nabla f_{j}(w^t) \Vert^2} \nonumber\\
      &=  \frac{s_j^t(\delta+\lambda \Vert \nabla f_{j}(w^t) \Vert^2) +   \delta f_{j}(w^t)  -  \delta s_{j}^t -  \delta\lambda \Vert \nabla f_{j}(w^t) \Vert^2 }{\delta+\lambda \Vert \nabla f_{j}(w^t) \Vert^2} \nonumber\\
    &=  \frac{\delta  f_{j}(w^t) + \lambda(s_{j}^t -  \delta )\Vert \nabla f_{j}(w^t) \Vert^2 }{\delta+\lambda \Vert \nabla f_{j}(w^t) \Vert^2}.
\end{align*}

Since the function $f_j$ is  $L_j$--smooth and $L_{\max}\geq L_j$ for all $j\in[n]$, we have by Lemma~\ref{lem:convsmoothinter} and with $\lambda \leq \frac{1}{2L_{\max}}$ that
\begin{eqnarray*}
    \delta(f_j(w^t) - \inf f_j) & \geq & \frac{\delta}{2 L_{\max} }\Vert \nabla f_j(w^t) \Vert^2  \\
    & \geq &
    \delta \lambda \Vert \nabla f_j(w^t) \Vert^2 \notag \\
    &=&
    \lambda(\delta - s_j^t) \Vert \nabla f_j(w^t) \Vert^2 + \lambda s_j^t \Vert \nabla f_j(w^t) \Vert^2 \notag\\
    & \geq &
    \lambda(\delta - s_j^t)\Vert \nabla f_j(w^t) \Vert^2 + \lambda(\inf f_j) \Vert \nabla f_j(w^t) \Vert^2\notag,
\end{eqnarray*}
where we used the induction assumption $s_j^t\geq \inf f_j$ in the last step.
The above can be rearranged into
\begin{align*}
    s_j^t+\delta\tau_t-\delta = \frac{\delta  f_{j}(w^t) + \lambda(s_{j}^t -  \delta )\Vert \nabla f_{j}(w^t) \Vert^2 }{\delta+\lambda \Vert \nabla f_{j}(w^t) \Vert^2} \geq \inf f_j.
\end{align*}
We conclude 
\[s_j^{t+1} = (1-\gamma)s_j^t+\gamma (s_j^t+\delta\tau_t-\delta) \geq (1-\gamma)\inf f_j + \gamma \inf  f_j = \inf f_j.\]
Finally, note again that for $\gamma =1$ we have $\hat s_j^{t+1} = s_j^{t+1}$ and the above arguments hold verbatim, and thus $\hat s_j^{t+1} \geq \inf f_j.$ 
\end{proof}

\if{It can be interesting to rewrite this algorithm in an implicit form (to see that it is implicit, look at the dependency between $\tau_i^t$ and $s_i^{t+1}$):
\begin{center}
    $\begin{cases}
    \tau_i^t = \frac{f_i(w^t) - s_i^{t+1}}{\Vert \nabla f_i(w^t) \Vert^2} \\
    w^{t+1} = w^t - \tau_i^t \nabla f_i(w^t) \\
    s_i^{t+1} = s_i^t + \tau_i^t - \lambda
    \end{cases}$
\end{center}}\fi

\subsection{Properties of Surrogate Function}

\begin{lemma}[Lower bounds for the surrogate]\label{L:lower bouds phiit}
Let Assumption \ref{ass:convexity} hold. Let $w \in \mathbb{R}^d$, $t \in \mathbb{N}$. For all $i \in \{1, \dots, n\}$ it holds
\begin{align}
    \inf\limits_{s \in \mathbb{R}^n} \phi_{i,t}(w,s) &= f_i(w) + \frac{\delta}{2} -\frac{\lambda}{2}\Vert \nabla f_i(w^t) \Vert^2, \label{eq:minphi1}
     \\ 
    \inf\limits_{w \in \R^d, s\in \R^n} \phi_{i,t}(w,s) &= \inf f_i + \frac{\delta}{2} - \frac{\lambda}{2}\Vert \nabla f_i(w^t) \Vert^2,\label{eq:minphi2}\\
    \inf\limits_{s \in \mathbb{R}^n} \phi_{t}(w,s) &= f(w) + \frac{\delta}{2} -\frac{1}{n}\sum\limits_{i=1}^n \frac{\lambda}{2}\Vert \nabla f_i(w^t) \Vert^2,\label{eq:minphi3}\\
    \inf\limits_{w \in \R^d, s\in \R^n}\phi_{t}(w,s) &= \inf f  + \frac{\delta}{2} -\frac{1}{n}\sum\limits_{i=1}^n \frac{\lambda}{2}\Vert \nabla f_i(w^t) \Vert^2.\label{eq:minphi4}
\end{align}

\end{lemma}

\begin{proof}
\noindent {\bf Proof of~\eqref{eq:minphi1}.} Let $w \in \mathbb{R}^d$ be fixed. By Lemma \ref{L:convexity surrogate}, $\phi_{i,t}(w, \cdot)$ is convex. The stationarity condition can be derived by \eqref{eq:gradphiit} and is fulfilled for any $s^*_t \in \mathbb{R}^n$ such that $s_{i,t}^*:=(s_t^*)_i = f_i(w) - \lambda \Vert \nabla f_i(w^t) \Vert^2$. We deduce that $s^*_t$ is a minimizer of $\phi_{i,t}(w, \cdot)$, and plugging in yields
\begin{eqnarray*}
    \inf\limits_{s \in \mathbb{R}^n} \phi_{i,t}(w,s) 
    &= &
    \phi_{i,t}(w,s^*_t)
    =
     s_{i,t}^* + \frac{1}{2} \frac{(f_i(w) - s_{i,t}^* + \delta)_+^2}{\delta + \lambda\Vert \nabla f_{i}(w^t) \Vert^2} \\
    &=&
     f_i(w) - \lambda \Vert \nabla f_i(w^t) \Vert^2 + \frac{1}{2} (\delta+ \lambda\Vert \nabla f_i(w^t) \Vert^2)\\
    &=&f_i(w) + \frac{\delta}{2} -\frac{\lambda}{2}\Vert \nabla f_i(w^t) \Vert^2.
\end{eqnarray*}

{\bf Proof of~\eqref{eq:minphi2}.}
Let $w_i^* \in {\rm{argmin}}_w~f_i(w)$, and define $s^*_t$ as above.
From Lemma \ref{L:common minimizers}, we have that 
\begin{equation*}
    \inf\limits_{w \in \R^d, s\in \R^n} \phi_{i,t}(w,s) = \phi_{i,t}(w_i^*,s^*_t) =\inf f_i +\frac{\delta}{2}  -\frac{\lambda}{2}  \Vert \nabla f_i(w^t) \Vert^2.
\end{equation*}

The proofs of~\eqref{eq:minphi3} and~\eqref{eq:minphi4} follow from the same arguments above.

\end{proof}

As a consequence of the previous lemma we have that
\begin{lemma}\label{lem:lowerphi}
Let Assumptions \ref{ass:convexity} and \ref{ass:smoothness} hold.
Let
\begin{equation} \label{eq:sigma}
    \sigma \; := \; \inf f - \frac{1}{n} \sum_{i=1}^n \inf f_i \;\geq \; 0.
\end{equation} 
It holds
\begin{equation}
    \inf\limits_{s \in \mathbb{R}^n} \phi_{t}(w^t,s) \geq \inf f +
    \frac{\delta}{2} - \lambda L_{\max}\sigma
    +(1- \lambda L_{\max} )(f(w^t) - \inf f).
\end{equation} 
\end{lemma}
\begin{proof}
First note that $\sigma$ in~\eqref{eq:sigma} is positive because
 \[\inf f = \inf \frac{1}{n} \sum_{i=1}^n f_i \geq \frac{1}{n} \sum\limits_{i=1}^n \inf f_i.\]
Using smoothness, we have from Lemma~\ref{lem:convsmoothinter} and specifically~\eqref{eq:convsmoothinterE}
that
\begin{equation*}
    \mathbb{E}_t[\Vert \nabla f_i(w^t) \Vert^2]
    =
    \frac{1}{n}\sum\limits_{i=1}^n \Vert \nabla f_i(w^t) \Vert^2
    \leq 2 L_{\max}(f(w^t) - \inf f+ \sigma).
\end{equation*}
Using the above in~\eqref{eq:minphi3} we have that
\begin{eqnarray*}
    \inf\limits_{s \in \mathbb{R}^n} \phi_{t}(w^t,s)
    &\geq &
    f(w^t) + \frac{\delta}{2} - \frac{\lambda}{2} 2 L_{\max}\left(f(w^t) - \inf f + \sigma \right) \\
    &=&
    f(w^t)
    + \frac{\delta}{2} 
    - \lambda L_{\max}(f(w^t) - \inf f)
    - \lambda  L_{\max} \sigma \\
    &=&
    \inf f +
    \frac{\delta}{2} - \lambda  L_{\max} \sigma
    +(1- \lambda L_{\max} )(f(w^t) - \inf f).
\end{eqnarray*}
\end{proof}

\begin{lemma}[From the surrogate to $f$]\label{L:surrogate lower bound}
Let Assumption \ref{ass:convexity} hold and let $t \in \mathbb{N}$, let $w^* \in {\rm{argmin}}~f$. Let $s_i^*=f_i(w^*)$ and $z^*=(w^*,s^*)$ and $z^t=(w^t,s^t)$.
\begin{enumerate}
    \item\label{L:surrogate lower bound:lipschitz} If Assumption \ref{ass:lipschitz} holds,
    then
    \begin{equation*}
        f(w^t) - \inf f \leq \phi_t(z^t) - \phi_t(z^*) + \frac{\lambda G^2}{2n}.
    \end{equation*}
    \item\label{L:surrogate lower bound:smooth} If Assumption \ref{ass:smoothness} holds, and defining $\nu := \inf f - \E{\inf f_i}$,
    then
\begin{equation*}
    (1-\lambda L_{\max})(f(w^t)-f(w^*)) \leq 
    \phi_t(z^t) - \phi_t(z^*) + \lambda L_{\max}\nu.
\end{equation*}
\end{enumerate}
\end{lemma}
\begin{proof}
Item \ref{L:surrogate lower bound:lipschitz}: Lipschitz continuity of $f_i$ implies that $\|\nabla f_i(w)\|\leq G_i$ for all $w\in\mathbb{R}^d$. Using Lemma \ref{L:lower bouds phiit}, and the fact that $\lambda G^2 \leq \delta$, we have
\begin{equation}
    \phi_t(z^t) \geq f(w^t) +\frac{\delta}{2} - \frac{1}{n}\sum_{i=1}^n \frac{\lambda}{2}\|\nabla f_i(w^t)\|^2 
    \geq f(w^t) +\frac{\delta}{2}  -\frac{\lambda}{2n}\sum_{i=1}^n G_i^2
\end{equation}
Moreover, we can compute from the definition of $\phi_t$ and $z^*$ that
\begin{equation} \label{eq:upper-bound-phit-star}
    \phi_t(z^*) = \frac{1}{n}\sum_{i=1}^n \left( \frac{1}{2}\frac{\delta^2}{\delta+\lambda \|\nabla f_i(w^t)\|^2} + f_i(w^*)  \right)
    \leq \frac{\delta}{2} + f(w^*)
    = \frac{\delta}{2} + \inf f.
\end{equation}
Combining the above two equations, we finally obtain that
\begin{equation*}
    \phi_t(z^t) - \phi_t(z^*)  + \frac{\lambda G^2}{2n}\geq f(w^t) - \inf f .
\end{equation*}
Item \ref{L:surrogate lower bound:smooth}: From the smoothness and Lemma \ref{lem:convsmoothinter}, we have $\|\nabla f_i(w^t)\|^2 \leq 2 L_{\max}(f_i(w^t)-\inf f_i)$. 
Thus we can use Lemma \ref{L:lower bouds phiit} as before to write
\begin{equation}
    \phi_t(z^t) \geq f(w^t) +\frac{\delta}{2} - \frac{1}{n}\sum_{i=1}^n \frac{\lambda}{2}\|\nabla f_i(w^t)\|^2 \geq (1-\lambda L_{\max})f(w^t) +\frac{\delta}{2} + \lambda L_{\max}\E{\inf f_i}.
\end{equation}
Using again \eqref{eq:upper-bound-phit-star}, namely
\begin{equation*}
    \phi_t(z^*) \leq \frac{\delta}{2} + f(w^*) = \frac{\delta}{2} +(1-\lambda L_{\max})\inf f + \lambda L_{\max} \inf f.
\end{equation*}
we get
\begin{equation*}
    \phi_t(z^t) - \phi_t(z^*) \geq (1-\lambda L_{\max})(f(w^t)-f(w^*)) + \lambda L_{\max}(\E{\inf f_i} - \inf f).
\end{equation*}
\end{proof}
The next lemma shows that $\phi_{i,t}$ satisfies a property which is typically verified by convex functions with $1$--smooth  gradients.
This loosely justifies why it is legitimate to take a stepsize lower or equal to $1$ in Lemma \ref{L:online SGD equivalence}.

\begin{lemma}[Gradient bound]\label{L:smoothness phiit}
For every $t \in \mathbb{N}$, 
\begin{equation}\label{eq:gradphiboundeq}
   \frac{1}{2} \Vert \nabla \phi_{{j_t},t}(w^t,s^t) \Vert_{\mD^{-1}}^2 
   =
   \phi_{{j_t},t}(w^t,s^t) 
   - \hat s^{t+1}_{j_t} - \frac{\delta}{2},
\end{equation}
where 
\begin{equation}\label{eq:shat}
    \hat s_{j_t}^{t+1} = s_{j_t}^t +\delta(\tau_t - 1)
\end{equation}  and 
$\mD$ is defined in \eqref{D:metric}.

\end{lemma}

\begin{proof}
Let $i:={j_t}$.
Using the expression of $\nabla \phi_{i,t}(w^t,s^t)$ in \eqref{eq:gradphiit}, together with the definition of $\phi_{i,t}$ \eqref{eq:L1sgd}, we can compute the following:  
\begin{eqnarray*}
\Vert \nabla \phi_{i,t}(w^t,s^t) \Vert_{\mD^{-1}}^2
& = &
 \lambda\frac{(f_{i}(w^t) - s^t_i + \delta)_+^2}{(\delta + \lambda \Vert \nabla f_{i}(w^t) \Vert^2)^2} \Vert \nabla f_{i}(w^t) \Vert^2 
 + 
 \delta \left( \frac{(f_{i}(w^t) - s^t_i + \delta)_+}{\delta + \lambda\Vert \nabla f_{i}(w^t) \Vert^2} - 1\right)^2 \\
 & = &
  \frac{(f_{i}(w^t) - s^t_i + \delta)_+^2}{(\delta + \lambda\Vert \nabla f_{i}(w^t) \Vert^2)^2} \left(\lambda \Vert \nabla f_{i}(w^t) \Vert^2 + \delta \right) 
  -2 \delta \frac{(f_{i}(w^t) - s^t_i + \delta)_+}{\delta + \lambda\Vert \nabla f_{i}(w^t) \Vert^2} 
  + \delta \\
 & = &
  \frac{(f_{i}(w^t) - s^t_i + \delta)_+^2}{\delta + \lambda\Vert \nabla f_{i}(w^t) \Vert^2}  
  -2 \delta \frac{(f_{i}(w^t) - s^t_i + \delta)_+}{\delta + \lambda\Vert \nabla f_{i}(w^t) \Vert^2} 
  + \delta \\
    & = &
    2\phi_{i,t}(w^t,s^t_i)  
    - 2 s_i^t
    -2\delta \frac{(f_{i}(w^t) - s^t_i + \delta)_+}{\delta + \lambda\Vert \nabla f_{i}(w^t) \Vert^2} 
    + \delta \\
    &=&
    2\phi_{i,t}(w^t,s^t_i) 
    - 2 \left( 
        s_i^t 
        + \delta (\tau_t - \tfrac{1}{2})
    \right) \\
        \\
     &= & 
     2\phi_{i,t}(w^t,s^t_i) 
     - 2 \left( s_i^{t}+ \frac{\delta}{2} +\frac{1}{\gamma}(s_i^{t+1}-s_i^t)\right)
    \\
     &= & 
     2\phi_{i,t}(w^t,s^t_i) 
     - 2 (\hat s_i^{t+1}+ \frac{\delta}{2})
    ,\end{eqnarray*}
    where in the last equality we used~\eqref{alg:relaxedversion}.

\end{proof}

Finally we show that if $f_i$'s are convex,  then the surrogate functions  $\phi_{i,t}$ are convex.
\begin{lemma}[Convexity of the surrogate]\label{L:convexity surrogate}
Let Assumption \ref{ass:convexity} hold. Then, 
 $\phi_{i,t}$ is convex for every $t \in \mathbb{N}$ and $i \in \{1, \dots, n\}$ and it holds
 \begin{align*}
     \phi_{i,t}(z^*) - \phi_{i,t}(z) - \langle \nabla \phi_{i,t}(z), z^* - z \rangle \geq 0, \quad \forall z,z^* \in \mathbb{R}^{d+n},
 \end{align*}
where $z=(w,s)$ and $z^*=(w^*,s^*)$. Consequently $\phi_t$ is convex and it holds
\begin{align*}
    \phi_{t}(z^*) - \phi_{t}(z) - \langle \nabla \phi_{t}(z), z^* - z \rangle  \geq 0, \quad \forall z,z^* \in \mathbb{R}^{d+n}.
\end{align*}
\end{lemma}

\begin{proof}
Let $t \in \mathbb{N}$ and $i \in \{1, \dots, n\}$ be fixed.
Since $\phi_{i,t}$ is differentiable, it is enough to show that
\begin{equation*}
    \phi_{i,t}(z^*) - \phi_{i,t}(z) - \langle \nabla \phi_{i,t}(z), z^* - z \rangle \geq 0, \quad \forall z,z^* \in \mathbb{R}^{d+n}.
\end{equation*}
Let $z,z^*$ be fixed, and let us call $Q \in \mathbb{R}$ the quantity in the left-hand side of this inequality.
From the definition of $\phi_{i,t}$, the value of its gradient \eqref{eq:gradphiit}, and with the introduction of the notation $r(z):=f_i(w) - s_i + \delta$, we can rewrite $Q$ as
\begin{align*}
Q &=
    s_i^* 
    + \frac{1}{2}\frac{r(z^*)^2_+}{\delta+\lambda\Vert \nabla f_i(w^t)\Vert^2}
    -s_i
    - \frac{1}{2}\frac{r(z)^2_+}{\delta+\lambda\Vert \nabla f_i(w^t)\Vert^2} \\
    & \quad 
    - \frac{r(z)_+}{\delta+\lambda\Vert \nabla f_i(w^t)\Vert^2} \langle \nabla f_i(w), w^* - w \rangle 
    - \left(1 - \frac{r(z)_+}{\delta+\lambda\Vert \nabla f_i(w^t)\Vert^2}\right)(s^*_i -s_i).
\end{align*}
Reordering the terms, and cancelling the $(s^* -s)$ terms, we obtain
\begin{equation*}
    Q
    =
    \frac{1}{2}\frac{r(z^*)^2_+}{\delta+\lambda\Vert \nabla f_i(w^t)\Vert^2}
    +
    \frac{r(z)_+}{\delta+\lambda\Vert \nabla f_i(w^t)\Vert^2}
    \left( 
    s^*_i -s_i - \langle \nabla f_i(w), w^* - w \rangle  - \frac{1}{2}r(z)_+
    \right).
\end{equation*}
Convexity of $f_i$ yields 
$- \langle \nabla f_i(w), w^* - w \rangle \geq f_i(w) - f_i(w^*) $.
Thus, as $r(z)_+ \geq 0$, we have
\begin{equation}\label{e:cosu1}
    Q
    \geq 
    \frac{1}{2}\frac{r(z^*)^2_+}{\delta+\lambda\Vert \nabla f_i(w^t)\Vert^2}
    +
    \frac{r(z)_+}{\delta+\lambda\Vert \nabla f_i(w^t)\Vert^2}
    \left( 
    s^*_i -s_i +f_i(w) - f_i(w^*) - \frac{1}{2}r(z)_+
    \right).
\end{equation}
Now we consider two cases.
First, if $f_i(w) - s_i + \delta \leq 0$, then $r(z) \leq 0$, which means that $r(z)_+ =0$, and thus from \eqref{e:cosu1} we have that $Q \geq 0$.\\
On the other hand, if $f_i(w) - s_i + \delta > 0$, then $r(z) > 0$ and $r(z)_+ = r(z) = f_i(w) - s_i + \delta$.
Therefore, we can write
\begin{equation*}
    s^*_i -s_i +f_i(w) - f_i(w^*)  - \frac{1}{2}r(z)_+
    =
     r(z)_+ -r(z^*)  - \frac{1}{2}r(z)_+
    =
      \frac{1}{2}r(z)_+  -r(z^*)
    \geq
      \frac{1}{2}r(z)_+  -r(z^*)_+,
\end{equation*}
where in the last inequality we used the fact that for any real number $r \in \mathbb{R}$, $r_+ \geq r$.
We can now apply this last result to  \eqref{e:cosu1}, and conclude that
\begin{align*}
    Q
    &\geq 
    \frac{1}{2}\frac{r(z^*)^2_+}{\delta+\lambda\Vert \nabla f_i(w^t)\Vert^2}
    +
    \frac{r(z)_+}{\delta+\lambda\Vert \nabla f_i(w^t)\Vert^2}
    \left( 
    -r(z^*)_+ + \frac{1}{2}r(z)_+ 
    \right) \\
    &=
    \frac{1}{2}
    \frac{1}{\delta+\lambda\Vert \nabla f_i(w^t)\Vert^2}
    \left(  
    r(z^*)_+
    -
    r(z)_+
    \right)^2   \\
    &\geq 0.
\end{align*}
Summing the above over $i=1,\dots,n$ and dividing by $n$ gives the stated result on $\phi_t$.
\end{proof}

\subsection{Lyapunov estimates}

Now we develop several Lyapunov functions and their bounds.
\begin{lemma}[Estimates on the iterates]\label{L:lyapunov estimate}
Let Assumption \ref{ass:convexity} hold and let $z^* \in \mathbb{R}^d \times \mathbb{R}^n$, and $t \in \mathbb{N}$. It follows that
\begin{align}\label{eq:conv1ststep}
   \E{\norm{z^{t+1} -z^*}_{\mD}^2} - \E{\norm{z^{t} -z^*}_{\mD}^2}
   & \leq \;
    2\gamma(1-\gamma) 
     \mathbb{E}[\phi_{t}(z^*) -  \phi_{t}(z^t)]
     +2\gamma^2 \bar \sigma_t^*,\\
\label{eq:lyapunov estimate:iterates z_t}
    \E{\norm{z^{t+1} -z^*}_{\mD}^2}-\E{\norm{z^{t} -z^*}_{\mD}^2}  &\leq -2\gamma \E{\phi_t(z^t) - \phi_t(z^*)}  +2\gamma^2 \bar \sigma_t,
\end{align}
where 
\begin{equation}\label{eq:sigmat}
    \bar \sigma_t^* \eqdef  \E{\phi_t(z^*) } - \E{\hat s^{t+1}_{j_t} }- \frac{\delta}{2} 
    \quad \text{ and } \quad 
    \bar \sigma_t \eqdef
    \E{\phi_t(z^t) } - \E{\hat s^{t+1}_{j_t} }- \frac{\delta}{2}.
\end{equation}
Moreover, assume now that $w^* \in {\rm{argmin}}~f$ and $s_i^* = f_i(w^*)$.
\begin{enumerate}
    \item\label{L:lyapunov estimate:sigma_t raw} We have 
    \begin{equation*}
    \bar \sigma_t \leq\frac{\delta}{2} +  \frac{1}{2n\delta} \sum_{i=1}^n \E{(f_{i}(w^t) - s_i^t)^2}.
\end{equation*}
    \item\label{L:lyapunov estimate:sigma_t lipschitz} If Assumption \ref{ass:lipschitz} holds and $\lambda G^2 \leq \delta$, then
    \begin{equation*}
    \bar \sigma_t \leq
    \frac{\delta}{2} + 
    \frac{1}{n}
    \Vert z^t - z^* \Vert^2_{\mD}.
\end{equation*}
    \item\label{L:lyapunov estimate:sigma_t^* smooth} If Assumption \ref{ass:smoothness} holds and $2 \lambda L_{\max} \leq 1$,
    then
$$ \bar \sigma_t^* \leq \sigma^* :=  \inf f - \mathbb{E}[\inf f_i].$$
    \item\label{L:lyapunov estimate:sigma_t^* smooth interpolation} If Assumption \ref{ass:smoothness} holds and $2 \lambda L_{\max} \leq 1$, and interpolation~\eqref{eq:interpolation} holds, then $\bar \sigma_t^* \leq 0$.
\end{enumerate}

\end{lemma}
\begin{proof}
Let $z^* \in \mathbb{R}^d \times \mathbb{R}^n$ and $t \in \mathbb{N}$.
From Theorem~\ref{theo:SGDVM} we have that
\begin{eqnarray}\label{eq:SGDVMmain}
\E{\norm{z^{t+1} -z^*}_{\mD}^2}-\E{\norm{z^{t} -z^*}_{\mD}^2} \; \leq \;-2\gamma \E{\phi_t(z^t) - \phi_t(z^*)} + \gamma^2 \E{\norm{\nabla \phi_{j_t,t}(z^t)}_{\mD^{-1}}^2}.
\end{eqnarray}
Using~\eqref{eq:gradphiboundeq} gives
\begin{align}
\E{\norm{z^{t+1} -z^*}_{\mD}^2}-\E{\norm{z^{t} -z^*}_{\mD}^2} & \leq -2\gamma \E{\phi_t(z^t) - \phi_t(z^*)}  +2\gamma^2\left(\E{\phi_t(z^t) } - \E{\hat s^{t+1}_{j_t} }- \frac{\delta}{2}\right) \nonumber \\
&=  -2\gamma(1-\gamma) \E{\phi_t(z^t) - \phi_t(z^*)}  +2\gamma^2\left(\E{\phi_t(z^*) } - \E{\hat s^{t+1}_{j_t} }- \frac{\delta}{2}\right)
\label{eq:SGDVMmain2}
\end{align}
The above inequality proves \eqref{eq:conv1ststep} and \eqref{eq:lyapunov estimate:iterates z_t}.

Now we turn to the proof of items 1-4. For this proof we use the less ambiguous notation for the step size $\tau_t$ in~\eqref{alg:relaxedversion} by explicating the dependency on the data point, that is 
$$ \tau_{i}^{t} := \tau_t =   \frac{(f_{i}(w^t)  -s_{i}^t+ \delta)_+}{\delta+\lambda \Vert \nabla f_{i}(w^t) \Vert^2}.$$

\begin{enumerate}
\item
Starting with 
  \begin{align}
     {\sigma}_{t} & \eqdef {\phi_{t}(z^t) } - {\hat s^{t+1}_{j_{t}} }- \frac{\delta}{2},
  \end{align}  
such that $\bar \sigma_t = \E{\sigma_t}$.
Now from~\eqref{eq:shat}  we have that
\begin{align*}
  \EE{t}{\hat s^{t+1}_{j_{t}}} &=     \EE{t}{s_{j_{t}}^{t} +\delta 
   ({\tau_{j_{t}}^{t} - 1}) }= \frac{1}{n} \sum_{i=1}^n {(s_{i}^{t}+\delta\tau_{i}^{t} )} - \delta.
\end{align*}
We also have that
\begin{align*}
   \phi_{i,t}(w^t,s_i^t) &= \frac{1}{2} \frac{(f_{i}(w^t) - s_i^t+ \delta)_+^2}{\delta +  \lambda\Vert \nabla f_{i}(w^t) \Vert^2} +  s_i^t  = \frac{1}{2}(f_{i}(w^t) - s_i^t+ \delta)_+\tau^t_i +s_i^t.
\end{align*}
Consequently,
\begin{align*}
  { \phi_{t}(w^t,s^t)} 
   &= \frac{1}{2n} \sum_{i=1}^n {(f_{i}(w^t) - s_i^t+ \delta)_+\tau^t_i} +\frac{1}{n} \sum_{i=1}^n {s_i^t}.
\end{align*}
Thus
\begin{align*}
  \EE{t}{\sigma_{t}} &= \frac{1}{2n} \sum_{i=1}^n {(f_{i}(w^t) - s_i+ \delta)_+\tau_i^t} +\frac{1}{n} \sum_{i=1}^n {s_i^t} - \frac{1}{n} \sum_{i=1}^n {(s_{i}^{t}+\delta\tau_{i}^{t}) } + \delta -\frac{\delta}{2} \\
    &=\frac{1}{2n} \sum_{i=1}^n {((f_{i}(w^t) - s_i^t+ \delta)_+ -2\delta)\tau^t_i}+\frac{\delta}{2}.
\end{align*}
Now we can consider the set $I_t = \{ i \;: \; f_i(w^t) -s_i^t +\delta>0 \}$ and simplify the above since 
\begin{align*}
   \EE{t} {\sigma_{t} }
    &=\frac{\delta}{2} + \frac{1}{2n} \sum_{i\in I_t} {(f_{i}(w^t) - s_i^t-\delta)\tau^t_i} \\
    &=\frac{\delta}{2} + \frac{1}{2n} \sum_{i\in I_t} {\frac{(f_{i}(w^t) - s_i^t-\delta)(f_{i}(w^t) - s_i^t+\delta)}{\delta + \lambda \norm{\nabla f_i(w^t)}^2}} \\
    &= \frac{\delta}{2} + \frac{1}{2n} \sum_{i\in I_t} {\frac{(f_{i}(w^t) - s_i^t)^2-\delta^2}{\delta + \lambda \norm{\nabla f_i(w^t)}^2}}.
\end{align*}
From this we directly obtain that 
\begin{equation*}
    \bar \sigma_t \leq\frac{\delta}{2} +  \frac{1}{2n\delta} \sum_{i=1}^n \E{(f_{i}(w^t) - s_i^t)^2},
\end{equation*}
which proves item \ref{L:lyapunov estimate:sigma_t raw}.
\item
Assume now that $f_i$ is $G_i$-Lipschitz, recall $G = \sqrt{\sum_{i=1}^n G_i^2}$ and suppose that $\lambda G^2 \leq \delta$.
Using the fact that $f_i(w^*) = s_i^*$, adding and subtracting $s_i^*$, gives
\begin{eqnarray*}
    \frac{1}{2n\delta} \sum_{i=1}^n (f_{i}(w^t) - s_i^t)^2
    & \leq &
    \frac{1}{n\delta} \sum_{i=1}^n (f_{i}(w^t) - f_i(w^*))^2 + (s_i^* - s_i^t)^2 \\
    & \leq &
    \frac{1}{n\delta} \sum_{i=1}^n G_i^2 \Vert w^t - w^* \Vert^2 +
    \frac{1}{n\delta} \sum_{i=1}^n (s_i^* - s_i^t)^2 \\
    & = &
    \frac{G^2}{n\delta}  \Vert w^t - w^* \Vert^2 +
    \frac{1}{n\delta} \Vert s^t - s^* \Vert^2 \\
    & \leq &
    \frac{1}{n\lambda}  \Vert w^t - w^* \Vert^2 +
    \frac{1}{n\delta} \Vert s^t - s^* \Vert^2 
    =
    \frac{1}{n} \Vert z^t - z^* \Vert^2_\mD.
\end{eqnarray*}

\item Let $w^* \in {\rm{argmin}}~f$, $f_i(w^*) = s_i^*$,  assume that $2\lambda L_{\max} \leq 1$, and consider the \emph{noise} term
\[\bar \sigma_t^* := \E{\phi_t(z^*) } - \E{\hat s^{t+1}_{j_t} }- \frac{\delta}{2}.\]
We can upper bound $\bar \sigma_t^*$ by observing that
\begin{eqnarray}
    \phi_{i,t}(z^*)
    &=&
    s_i^* + \frac{1}{2}\frac{(f_i(w^*) - s_i^* + \delta)^2_+}{\delta + \lambda \Vert \nabla f_i(w^t) \Vert^2} \nonumber \\
    &=&
    f_i(w^*)  + \frac{\delta}{2}\frac{ 1}{1 + \tfrac{\lambda}{\delta} \Vert \nabla f_i(w^t) \Vert^2} \\
    &\leq &
     f_i(w^*)  + \frac{\delta}{2} \label{eq:tempmlo9h8x4}
\end{eqnarray}
Taking the expectation conditioned to $t$, we obtain
\begin{equation*}
    \phi_{t}(z^*) \leq  \inf f + \frac{\delta}{2}.
\end{equation*}
Finally, using that $\mathbb{E}[\hat s^{t+1}_{i}] \geq \mathbb{E}[\inf f_i]$ (see Lemma \ref{L:lower bound slack variable}), we obtain that
\begin{align*}
    \bar \sigma_t^* & \leq  \inf f + \frac{\delta}{2} -\mathbb{E}[\hat s^{t+1}_{j_t}] -\frac{\delta}{2} \\
    & = \inf f -\mathbb{E}[\hat s^{t+1}_i] \\
    & \leq \inf f -\E{\inf f_i}.
\end{align*}

\item The final statement follows since under interpolation~\eqref{eq:interpolation} we have that  $ f_i(w^*) = \inf f_i.$ Consequently from the previous item we have that 
\[    \bar \sigma_t^* \leq  \inf f -\E{\inf f_i} = \inf f -\E{f_i(w^*) } = \inf f - \inf f =0.\]

\end{enumerate}
\end{proof}

\begin{lemma}[Upper bound - General]\label{L:upper bound phi raw}
Let Assumption \ref{ass:convexity} hold and let $w^* \in {\rm{argmin}}~f$, and $s_i^* = f_i(w^*)$.
For all $T \geq 1$:
\begin{equation}
    \frac{1}{T}
    \sum\limits_{t=0}^{T-1}
    \mathbb{E}[\phi_{t}(z^t) -  \phi_{t}(z^*)]
    \leq
    \frac{1}{2 \gamma T}
    \mathbb{E}[\norm{z^{0} -z^*}^2_{\mD}]
    + \frac{\gamma \delta}{2}
    + \frac{\gamma}{2 Tn\delta} 
    \sum\limits_{t=0}^{T-1}
    \sum_{i=1}^n \E{(f_{i}(w^t) - s_i^t)^2}.    
\end{equation}
\end{lemma}

\begin{proof}
Summing both sides of the inequality in \eqref{eq:lyapunov estimate:iterates z_t} from $t=0$ to $T-1$ gives
\[    \E{\norm{z^{T} -z^*}_{\mD}^2}-\E{\norm{z^{0} -z^*}_{\mD}^2}  \leq -2\gamma \sum_{t=0}^{T-1}\E{\phi_t(z^t) - \phi_t(z^*)}  +2\gamma^2 \sum_{t=0}^{T-1}\bar \sigma_t.\]
Using telescopic cancellation and dividing by $2 \gamma T$
gives
\begin{equation*}
    \frac{1}{T}
    \sum\limits_{t=0}^{T-1}
    \mathbb{E}[\phi_{t}(z^t) -  \phi_{t}(z^*)]
    \leq
    \frac{1}{2 \gamma T}
    \mathbb{E}[\norm{z^{0} -z^*}^2_{\mD}]
    + \frac{\gamma}{T} 
    \sum\limits_{t=0}^{T-1}
    \bar \sigma_t.    
\end{equation*}
Now using Lemma \ref{L:lyapunov estimate}.\ref{L:lyapunov estimate:sigma_t raw}  gives the desired bound.
\end{proof}

\subsection{Additional Non-smooth Lipschitz Convergence}

Here we provide some additional convergence results using the online SGD viewpoint for non-smooth Lipschitz functions.
\begin{theorem}[Upper bound - Lipschitz case]\label{L:upper bound phi lipschitz}
Let Assumptions \ref{ass:convexity} and \ref{ass:lipschitz} hold. Let $w^* \in {\rm{argmin}}~f$, and $s_i^* = f_i(w^*)$. If $\lambda G^2 \leq \delta$, then for all $T \geq 1$:
\begin{equation}\label{L:upper bound phi lipschitz:eq min phi}
    \min\limits_{t=0, \dots, T-1}
    \mathbb{E}[\phi_{t}(z^t) -  \phi_{t}(z^*)]
    \leq
    \frac{\gamma\norm{z^{0} -z^*}^2_{\mD}}{n(1 - \theta^T)}
    + \frac{\gamma \delta}{2},
    \quad \text{ with } \theta = \frac{n}{n+2 \gamma^2} \in ]0,1[.
\end{equation}
Moreover, for $\bar w^T:= \frac{1}{\sum_{t=1}^T \theta^t} \sum\limits_{t=0}^{T-1}\theta^{t+1}w^t$ it holds
\begin{equation}\label{L:upper bound phi lipschitz:eq f}
    \E{f(\bar w^T)-f(w^*)}
    \leq 
    \frac{\tfrac{\gamma}{\lambda}\norm{w^{0} -w^*}^2 +\tfrac{\gamma}{\delta}\norm{s^{0} -s^*}^2 }{n(1 - \theta^T)}
    + \frac{\delta}{2}(\gamma+\frac{1}{n}) .
\end{equation}
\end{theorem}

\begin{proof}
Let $\theta := \frac{n}{n+2 \gamma^2}$, and $\alpha_t := \theta^t$, such that $\alpha_t =  \alpha_{t+1}(1 + \frac{2 \gamma^2}{n})$.
Consider the inequality
\eqref{eq:lyapunov estimate:iterates z_t} and multiply it by $\alpha_{t+1}$ to obtain, after reordering the terms:
\begin{equation*}
    2\gamma \alpha_{t+1}\E{\phi_t(z^t) - \phi_t(z^*)} 
     \leq \alpha_{t+1}\E{\norm{z^{t} -z^*}_{\mD}^2} - \alpha_{t+1}\E{\norm{z^{t+1} -z^*}_{\mD}^2}  +2\gamma^2\alpha_{t+1} \bar \sigma_t.
\end{equation*}
Use Lemma \ref{L:lyapunov estimate}.\ref{L:lyapunov estimate:sigma_t lipschitz} to write
\begin{eqnarray*}
    2\gamma \alpha_{t+1}\E{\phi_t(z^t) - \phi_t(z^*)} 
     &\leq & \alpha_{t+1}(1 +\frac{2\gamma^2}{n}) \E{\norm{z^{t} -z^*}_{\mD}^2} - \alpha_{t+1}\E{\norm{z^{t+1} -z^*}_{\mD}^2}  +\gamma^2\alpha_{t+1}{\delta}
     \\
     & \leq & 
     \alpha_t\E{\norm{z^{t} -z^*}_{\mD}^2} - \alpha_{t+1}\E{\norm{z^{t+1} -z^*}_{\mD}^2}  +\gamma^2\alpha_{t+1}{\delta}.
\end{eqnarray*}
Sum this inequality to obtain (we use the fact that $\alpha^0=1$):
\begin{equation*}
    2\gamma \sum\limits_{t=0}^{T-1}\alpha_{t+1}\E{\phi_t(z^t) - \phi_t(z^*)}
    \leq 
    \E{\norm{z^{0} -z^*}_{\mD}^2}  +\gamma^2{\delta}\sum\limits_{t=0}^{T-1}\alpha_{t+1}.
\end{equation*}
Define $A_T := \sum\limits_{t=0}^{T-1}\alpha_{t+1}$.
It is easy to see that $A_T$ is positive and that $A_T = \theta \frac{1 - \theta^T}{1-\theta}$.
Dividing the above inequality by $A_T$, we finally obtain
\begin{equation}
\label{eq:lipschitz-bound-phi}
    2\gamma\frac{1}{A_T} \sum\limits_{t=0}^{T-1}\alpha_{t+1}\E{\phi_t(z^t) - \phi_t(z^*)}
    \leq 
    \frac{1}{A_T}
    \E{\norm{z^{0} -z^*}_{\mD}^2}  +\gamma^2{\delta}.
\end{equation}
To obtain \eqref{L:upper bound phi lipschitz:eq min phi}, take the minimum among the $\E{\phi_t(z^t) - \phi_t(z^*)}$, divide by $2 \gamma$, and compute
\begin{equation*}
    \frac{1}{A_T}
    =
    \frac{1-\theta}{\theta} \frac{1}{1 - \theta^T}
    =
    \frac{2\gamma^2}{n} \frac{1}{1 - \theta^T}.
\end{equation*}
To obtain \eqref{L:upper bound phi lipschitz:eq f}, use \eqref{eq:lipschitz-bound-phi} together with Lemma \ref{L:surrogate lower bound}.\ref{L:surrogate lower bound:lipschitz} and $\lambda G^2 \leq \delta$ to obtain
\begin{equation*}
    2\gamma\frac{1}{A_T} \sum\limits_{t=0}^{T-1}\alpha_{t+1}\E{f(w^t)-f(w^*)} - 2\gamma \frac{\delta}{2n}
    \leq 
    \frac{1}{A_T}
    \E{\norm{z^{0} -z^*}_{\mD}^2}  +\gamma^2{\delta} .
\end{equation*}
Move now the term $2\gamma \frac{\delta}{2n}$, and apply Jensen's inequality to get
\begin{equation*}
    \E{f(\bar w^T) - f(w^*)} \leq  \frac{1}{A_T}\sum\limits_{t=0}^{T-1}\alpha_{t+1}\E{f(w^t)-f(w^*)}
    \quad \text{ with } \quad
    \bar w^T \eqdef \frac{1}{A_T}\sum\limits_{t=0}^{T-1}\alpha_{t+1}w^t.
\end{equation*}
which allows to conclude.
\end{proof}

\begin{theorem}[Complexity - Lipschitz case]
In the context of Theorem~\ref{L:upper bound phi lipschitz}, if $\gamma \leq \sqrt{\frac{n\log(n)}{T}}$ then 
\begin{equation*}
      \E{f(\bar w^T)-f(w^*)}
    \leq \cO\left(\sqrt{\frac{n\log(n)}{T}}+\frac{1}{n}\right)
    \quad \text{ provided that } \quad
    T \geq 2\log(n).
\end{equation*}
In addition, if $\delta = \frac{G^2}{\sqrt{n}}$  then 
\begin{equation*}
      \E{f(\bar w^T)-f(w^*)}
    \leq \tilde{\cO}\left(\frac{1}{\sqrt{T}}\right)
    \quad \text{ provided that } \quad
   n^3 \geq  T \geq 2\log(n).
\end{equation*}
\end{theorem}
\begin{proof}
For ease of reference we recall the result from 
Theorem~\ref{L:upper bound phi lipschitz}, that is
\begin{equation}\label{eq:templi8hx4z4}
      \E{f(\bar w^T)-f(w^*)}
    \leq 
    \frac{\tfrac{\gamma}{\lambda}\norm{w^{0} -w^*}^2 +\tfrac{\gamma}{\delta}\norm{s^{0} -s^*}^2}{n(1 - \theta^T)}
    + \frac{\delta}{2}(\gamma+\frac{1}{n}).
   \end{equation} 
First we choose $T$ so that $\theta^T \leq \frac{1}{n}.$ That is,
\[ \log(n) \leq T\log(\frac{1}{\theta}).\]
Since $\theta<1$ we have that
\begin{equation}\label{eq:tmepmxo9h84}
     \frac{\log(n)}{\log(\frac{1}{\theta})} \leq T.
\end{equation}
Now using \[ \log(\tfrac{1}{\theta}) \geq 1-\theta\]
we have that~\eqref{eq:tmepmxo9h84} holds if 
\begin{equation}\label{eq:tmepmxo9h842}
     \frac{\log(n)}{1-\theta} =     \frac{\log(n)(n+2\gamma^2)}{2\gamma^2} =\log(n) \left(\frac{n}{2\gamma^2} +1\right) \leq T.
\end{equation}
In particular, 
if $\gamma \leq \sqrt{\frac{n\log(n)}{T}}$ then~\eqref{eq:tmepmxo9h842} holds if
\[T \geq  2\log(n).\]
With this constraint on $\gamma$ we have 
by~\eqref{eq:templi8hx4z4} that
\begin{align}
      \E{f(\bar w^T)-f(w^*)}
    &\leq 
    \sqrt{\frac{n\log(n)}{T}}\frac{\tfrac{1}{\lambda}\norm{w^{0} -w^*}^2 +\tfrac{1}{\delta}\norm{s^{0} -s^*}^2}{n-1}
    + \frac{\delta}{2}\left(\sqrt{\frac{n\log(n)}{T}}+\frac{1}{n}\right) \nonumber\\
    & \leq  \frac{1}{2}\sqrt{\frac{\log(n)}{n}\frac{1}{T}}\left(\tfrac{1}{\lambda}\norm{w^{0} -w^*}^2 +\tfrac{1}{\delta}\norm{s^{0} -s^*}^2\right)
     + \frac{\delta}{2}\left(\sqrt{\frac{n\log(n)}{T}}+\frac{1}{n}\right). \nonumber 
   \end{align} 
   This so far gives a complexity of $\cO\left(\sqrt{\frac{n\log(n)}{T}}+\frac{1}{n}\right).$
If $\delta = \frac{G^2}{\sqrt{n}}$ then we have from the above that
\begin{align}
      \E{f(\bar w^T)-f(w^*)}
    & \leq  \frac{1}{2}\sqrt{\frac{\log(n)}{T}}\left(\tfrac{1}{\lambda}\norm{w^{0} -w^*}^2 +\frac{1}{G^2}\norm{s^{0} -s^*}^2\right)
    + \frac{G^2}{2}\left(\sqrt{\frac{\log(n)}{T}}+\frac{1}{n^{\tfrac{3}{2}}}\right).\nonumber
   \end{align} 
  Thus if $T \leq n^3$ then the resulting complexity if $\cO\left(\sqrt{\frac{\log(n)}{T}} \right) = \tilde{\cO}\left(\frac{1}{\sqrt{T}} \right). $ 

\end{proof}

\section{SGD Convergence with Matrix Stepsize}

\begin{theorem} \label{theo:SGDVM}
Consider the problem of minimizing a sequence of functions given by
\begin{eqnarray}
\min_{z \in \R^d} \phi_t(z) = \frac{1}{n} \sum_{i=1}^n \phi_{i,t}(z).
\end{eqnarray}
Let $\mD \in \R^{p\times p}$ be a symmetric positive definite matrix.
Consider the online SGD method given by
\begin{eqnarray}\label{eq:SGDVMz}
z^{t+1} &=& z^t - \gamma \mD^{-1} \nabla \phi_{i_t,t}(z^t), 
\end{eqnarray}
where $i_t\in \{1,\ldots, n\}$ is sampled i.i.d. If $\phi_{t}(z)$ is convex around a given $z^*$, that is
\begin{equation}\label{temp:zlo9j3os9j}
    \phi_{t}(z^*) - \phi_{t}(z) \geq \langle \nabla \phi_{t}(z), z^* - z \rangle , \quad \forall z \in \mathbb{R}^{p}.
\end{equation}
then 
\begin{eqnarray}\label{eq:SGDVM}
\E{\norm{z^{t+1} -z^*}_{\mD}^2}-\E{\norm{z^{t} -z^*}_{\mD}^2} \; \leq \;-2\gamma \E{\phi_t(z^t) - \phi_t(z^*)} + \gamma^2 \E{\norm{\nabla \phi_{i_t,t}(z^t)}_{\mD^{-1}}^2}.
\end{eqnarray}
\end{theorem}
\begin{proof}
The iterates~\eqref{eq:SGDVMz} are equivalent to applying online SGD with the functions $h_{i,t}(y) \eqdef \phi_{i,t}(\mD^{-1/2} y)$ which is
\begin{eqnarray}
y^{t+1} = y^t - \gamma \nabla h_{i,t}(y^t), 
\end{eqnarray}
and then setting $z^t = \mD^{-1/2} y^t.$ We now proceed with a standard SGD proof. Expanding the squares we have that
\begin{eqnarray}
\norm{y^{t+1} -y^*}^2 &= & \norm{y^{t} -y^*}^2 -2\gamma \dotprod{\nabla h_{i_t,t}( y^k), y^t -y^*} + \gamma^2 \norm{\nabla h_{i_t,t}(y^t)}^2.\label{eq:lsh489h4o89}
\end{eqnarray}
Let $h_t(y) \eqdef \frac{1}{n} \sum_{i=1}^n h_{i,t}(z).$
Now note that since $\phi_t$ is convex around $z^*$ we have that $h_t$ is convex around $y^* = \mD^{1/2}z^*.$ Indeed, this follows from~\eqref{temp:zlo9j3os9j} together with 
\[ \langle \nabla h_{t}(y), y^* - y \rangle = \langle \mD^{-1/2}\nabla \phi_{t}(z), \mD^{1/2}(z^* - z) \rangle
= \langle \nabla \phi_{t}(z), z^* - z\rangle.\]
Consequently we have that 
\begin{eqnarray}
h_t(y^*) - h_t(y) \geq \langle \nabla h_{t}(y), y^* - y \rangle , \quad \forall y \in \R^p.
\end{eqnarray}
With this, and 
taking expectation conditioned on $y^t$ in~\eqref{eq:lsh489h4o89} we have that
\begin{eqnarray}
\EE{t}{\norm{y^{t+1} -y^*}^2} &= & \norm{y^{t} -y^*}^2 -2\gamma \dotprod{\nabla h_{t}( y^k), y^t -y^*} + \gamma^2 \EE{t}{\norm{\nabla h_{i_t,t}(y^t)}^2} \nonumber\\
&\leq & \norm{y^{t} -y^*}^2 -2\gamma (h_t(y) - h_t(y^*)) + \gamma^2 \EE{t}{\norm{\nabla h_{i_t,t}(y^t)}^2}.
\end{eqnarray}
Taking expectation again and re-arranging gives
\begin{eqnarray}
\E{\norm{y^{t+1} -y^*}^2}-\E{\norm{y^{t} -y^*}^2} \; \leq \;-2\gamma \E{h_t(y) - h_t(y^*)} + \gamma^2 \E{\norm{\nabla h_{i_t,t}(y^t)}^2}.
\end{eqnarray}
Changing the variables $y^t = \mD^{1/2}z^t,$ and $y^* = \mD^{1/2}z^*,$ gives the result~\eqref{eq:SGDVM}.
\end{proof}

\end{document}